\documentclass[11pt]{article}
\usepackage[margin=1in]{geometry}
\usepackage{times} % DO NOT CHANGE THIS
\usepackage{helvet} % DO NOT CHANGE THIS
\usepackage{courier} % DO NOT CHANGE THIS
\usepackage[hyphens]{url} % DO NOT CHANGE THIS
\usepackage{graphicx} % DO NOT CHANGE THIS
\usepackage{caption} % DO NOT CHANGE THIS
\usepackage{tikz-cd}
\usepackage{amssymb}
\usepackage{dsfont}
\usepackage{mleftright}
\usepackage{ntheorem}
\usepackage{mathtools}
\usepackage{algorithm, algpseudocodex}
\usepackage[colorlinks=true, allcolors=blue]{hyperref}
\usepackage{cleveref}
\usepackage{natbib}
\usepackage{multirow}
\usepackage{amsmath,amsfonts,graphicx}
\usepackage{xfrac}

\newtheorem{theorem}{Theorem}[section]
\newtheorem{lemma}[theorem]{Lemma}

\newtheorem{corollary}[theorem]{Corollary}
\newtheorem{definition}[theorem]{Definition}

\newenvironment{proof}{{\bf Proof:}}{\hfill\rule{2mm}{2mm}}
\newenvironment{proofsketch}{{\bf Proof sketch:}}{\hfill\rule{2mm}{2mm}}

\newcommand\E{\mathbb{E}}

\newcommand{\A}{\mathcal{A}}
\newcommand{\B}{\mathcal{B}}
\newcommand{\D}{\mathcal{D}}

\newcommand\asteriskfootnote[1]{%
  \begin{NoHyper}
  \renewcommand\thefootnote{* }\footnotetext{#1}%
  \end{NoHyper}
}

\newcommand{\calF}{\mathcal{F}}
\newcommand{\bbP}{\mathbb{P}}
\newcommand{\bbE}{\mathbb{E}}

\renewcommand{\dfrac}{\frac}

\newcommand{\ignore}[1]{}

\newcommand \sett[2] { \left.\left\{{#1}\right\vert{#2}\right\}}
\newcommand{\curly}[1]{ {\left\{ #1 \right\}}}
\newcommand{\roundy}[1]{ {\left( #1 \right)}}
\newcommand{\squary}[1]{ {\left[ #1 \right]}}

\newcommand{\obs}{\mathcal{O}}

\DeclarePairedDelimiterXPP\bPr[1]{\operatorname{\Pr}}{[}{]}{}{#1}

\newcommand{\mail}[1]{\href{mailto:#1}{\color{black} #1}}

\title{Delay as Payoff in MAB}

\author{ 
Ofir Schlisselberg$^{*}$ \\  Tel Aviv University \\ \mail{ofirs4@mail.tau.ac.il}
\and
Ido Cohen$^{*}$ \\ Tel Aviv University \\ \mail{idoc@mail.tau.ac.il}
\and
Tal Lancewicki \\  Tel Aviv University \\ \mail{lancewicki@mail.tau.ac.il}
\and
Yishay Mansour \\  Tel Aviv University and Google Research \\ \mail{mansour.yishay@gmail.com}
}

\begin{document}
\maketitle

\asteriskfootnote{ Equal contribution.}

\begin{abstract}

In this paper, we investigate a variant of the classical stochastic Multi-armed Bandit (MAB) problem, where the payoff received by an agent (either cost or reward) is both delayed, and directly corresponds to the magnitude of the delay. This setting models faithfully many real world scenarios such as the time it takes for a data packet to traverse a network given a choice of route (where delay serves as the agent’s cost); or a user's time spent on a web page given a choice of content (where delay serves as the agent’s reward). 
Our main contributions are tight upper and lower bounds for both the cost and reward settings. For the case that delays serve as costs, which we are the first to consider, we prove optimal regret that scales as $\sum_{i:\Delta_i > 0}\frac{\log T}{\Delta_i} + d^*$, where $T$ is the maximal number of steps, $\Delta_i$ are the sub-optimality gaps and $d^*$ is the \textit{minimal} expected delay amongst arms. For the case that delays serves as rewards, we show optimal regret of $\sum_{i:\Delta_i > 0}\frac{\log T}{\Delta_i} + \bar{d}$, where $\bar d$ is the second \textit{maximal} expected delay. These improve over the regret in the general delay-dependent payoff setting, which scales as $\sum_{i:\Delta_i > 0}\frac{\log T}{\Delta_i} + D$, where $D$ is the maximum possible delay. Our regret bounds highlight the difference between the cost and reward scenarios, showing that the improvement in the cost scenario is more significant than for the reward. Finally, we accompany our theoretical results with an empirical evaluation.

\end{abstract}

\section{Introduction}
Classical stochastic Multi-armed Bandit (MAB) is a well studied theoretical framework for sequential decision making, where at every step an agent chooses an action and immediately receives some payoff, be it reward or cost. A natural generalization of this framework considers the situation where the payoff is only received after a certain delay. This is known as the stochastic MAB problem with randomized delays \cite{jouliani}, and was extensively researched in previous work under various variants \cite{vernade,pike,zhou2019learning,gael,vernade2020linear,wu2022thompson,howson2023delayed,shi2023statistical}. In all these works the delay was considered reward-independent, namely, the reward and delay are sampled from independent distributions (more on this in the following sections). Later, \citet{tal} introduced reward-dependent delay, where the delay and reward are sampled from a joint distribution. This model is more challenging as it introduces selection bias into the observed payoffs. \citet{yifu} consider a special case of this they call strongly-dependent, where for all intents and purposes the delay is exactly the reward that we are trying to maximize. We second their motivation to study this case, and additionally generalize this to delay as cost that we are trying to minimize. This motivation is twofold. First, it models well many real world scenarios. Second, from a performance perspective, it offers a significant gain in the regret. 
We show that the \emph{delay as payoff} scenario is actually ``simpler" than the general payoff-dependent setting by providing a tighter upper bound compared to \cite{yifu}, and a significantly better bound for the cost setting.

To motivate the cost scenario, consider a communication network where we route packets from node $a$ to node $b$, and would like to do it in the fastest way possible. We can model this as a stochastic MAB, where every route $a\rightarrow b$ is an arm (action) and the time it takes the packet to arrive (formally known as Round Trip Time (RTT), see \citet{rtt}) is our payoff. For routing we want to minimize the \textit{RTT} and so we call the payoff \textit{``cost"}. 
To motivate the reward scenario, consider a web page with some dynamic content. We wish to capture the viewer's attention for as long as possible, 
by choosing the content wisely. A common metric used in advertising is Average Time on Page (ATP), used, for example, by Google Analytics.  In this case we want to maximize the \textit{ATP} and so we call the payoff \textit{``reward"}. In both scenarios the payoff (\textit{RTT} or \textit{ATP}) is the time elapsed from choosing an action until its payoff is final, hence it can be modeled as delay.

As far as performance, an immediate observation is that while the agent is waiting for an action's final payoff, it gains partial knowledge about its payoff as time progresses. More specifically, if at time $t_1$ an arm is played and by time $t_2$ the payoff has not been revealed, we can learn that the delay (hence the payoff) is at least $t_2-t_1$. This is a crucial observation that we use. Taking advantage of this knowledge, we can improve the regret bounds. Notice that this knowledge is one-sided in the sense that every time-step that passes provides an improved lower bound on the delay, but the same cannot be said for an upper bound. This is a challenging limitation that is the key difference between \textit{cost} and \textit{reward}, and explains how a better regret can be achieved for \textit{cost}. (This issue is further discussed in \Cref{cvd}.)

\begin{table*}[t!]
\centering
\renewcommand{\arraystretch}{2}
\begin{tabular}{||c|c|c||} 
 \hline
 & Reward & Cost \\ [0.5ex] 
 \hline\hline
 \cite{tal} & $\sum_{i:\Delta_i > 0}\frac{\log T}{\Delta_i} + D$& $\sum_{i:\Delta_i > 0}\frac{\log T}{\Delta_i} + D$ \\
 \hline
 \cite{yifu} & $ \sum_{i:\Delta_i > 0}\frac{\log T}{\Delta_i} + D\sum_{i:\Delta_i > 0}\Delta_i $ & N/A \\
 \hline
 This work & $ \sum_{i:\Delta_i > 0}{\frac{\log T}{\Delta_i}} + \min\{\bar{d},D\Delta_{max}\} $ & $ \sum_{i:\Delta_i > 0}{\frac{\log T}{\Delta_i}} +\min\{d^*,D\Delta_{max}\} $ \\ [1ex] 
 \hline
\end{tabular}

\caption{Pseudo regret comparison for works on delay-dependent payoff. $T$ is the total number of steps, $\Delta_i$ is the sub-optimality gap of arm $i$ and $\Delta_{max}=\max_i{\Delta_i}$, $D$ is the maximum possible delay, $\bar{d}$ is the second \textit{maximal} expected delay (across arms), and $d^*$ is \textit{minimal} expected delay.}
\label{table:1}
\end{table*}

\subsection{Our Contributions}
We study both reward and cost as payoff in the special case of payoff-dependent delay where delay serves as payoff. This setting presents an opportunity to use the partial knowledge accumulated while waiting for the payoff, to achieve better regret bounds. 
In order to conform with the literature, we normalize the payoff to be in $[0,1]$, by setting it to the actual delay divided by the maximum delay. This has no implications on the analysis, and is aimed to be inline with the existing regret bounds.
Our main contributions are the following:
\begin{enumerate}
    \item In the case of cost, we offer tight lower and upper bounds that scale as $ \sum_{i:\Delta_i > 0}{\frac{\log T}{\Delta_i}} + \min\{d^*,D\Delta_{max}\}$, where $d^*$ is the \textit{minimal} expected delay.
    \item In the case of rewards, we offer tight lower and upper bounds that scale as $ \sum_{i:\Delta_i > 0}{\frac{\log T}{\Delta_i}} + \min\{\bar{d},D\Delta_{max}\}$, where $\bar{d}$ is the second \textit{maximal} expected delay.
    Note that the cost regret bound can be significantly smaller than the reward regret bound, on the same problem instance.

    \item We complement our theoretical results in an experimental evaluation.  
\end{enumerate}

Our main results, along with a concise comparison
to previous work, are presented in \Cref{table:1}.
The two bounds provided, improve both on the general delay-dependent payoff setting which scales as $ \sum_{i:\Delta_i > 0}\frac{\log T}{\Delta_i} + D $ \cite{tal}\footnote{In \cite{tal} the additional additive term is formally the $(1-\Delta_{\min} / 4)$-quantile of the delay distribution. This can easily be as large as $D$, e.g., if there is a single arm with Bernoulli payoffs with $\mu(i) > 1/4$.}  and  $ \sum_{i:\Delta_i > 0}\frac{\log T}{\Delta_i} + D \sum_{i:\Delta_i > 0}\Delta_i $ by \cite{yifu}. Note that $D\Delta_{max}$ is already an improvement of factor $K$, but more significantly if $\bar{d} \ll D\Delta_{max}$ our bound is substantially lower. In the cost case this becomes even more clear as $d^*$ is potentially much smaller than $\bar{d}$.

\subsection{Cost vs Reward - intuition} 
\label{cvd}
The goal of a player in a MAB environment can be either to maximize his total payoff, in which case the payoff is called the ``reward", or to minimize his total payoff, then the payoff is called the ``cost". Normally when considering a MAB setting without delay, the choice of using cost or reward is interchangeable by simply changing the sign of the payoff, and most algorithms would be oblivious to the change. We argue that in the case of delayed payoff, where the payoff is the delay, cost and reward are very different. Specifically, when minimizing cost we can make better use of partial knowledge that we gain while waiting for payoff feedback.

We provide here some informal intuition for this, and we will make it more formal by providing lower bounds for both cases in \Cref{sec:clb,sec:rlb}. Consider a scenario where we have $K$ arms with constant delays (and thus, payoff) sorted from low to high $\{d_i\}_{i=1}^K$. If we maximize reward the best arm is $i_K$ with delay $d_K$. No matter how we play, arm $i_{K-1}$ and arm $i_K$ are indistinguishable until after $d_{K-1}$ time steps, simply because no feedback is received from either arm before $d_{K-1}$.
So the number of times we play sub-optimal arms depends on the second highest delay.
In comparison, when minimizing cost, the best arm is $i_1$ with delay $d_1$. After $d_1$ time steps we can already start getting some information about the cost of $i_1$, and so we can hope to stop playing sub-optimal arms as early as $O(d_1)$.

\paragraph{Paper organization}
The rest of the paper is organized as follows. In \Cref{sec:related work} we discuss related work and in \Cref{sec:problem setup} we formally present our settings. In \Cref{sec:cost} we present our main algorithm and analysis for the \textit{delay as cost} setting. In \Cref{sec:reward} we present our algorithm and results for the  \textit{delay as reward} setting. In \Cref{sec:experiments} we present empirical evaluation of our algorithms compared to previous related works. \Cref{sec:discussion} is a discussion.
Most of the proofs are deferred to the supplementary material.

\section{Related work}
\label{sec:related work}
The delayed payoff in MAB has recently gained significant attention. Most previous works have been devoted to \textit{payoff-independent} delays, often treating them as some unknown distribution. This line of work started with \cite{dudik} who introduced a constant delay, and offered a regret bound with linear dependence on the delay. \citet{jouliani} extended this to stochastic, yet bounded, delay.
Later variations include \citet{zhou2019learning}, who made a distinction between \textit{arm-dependent} and \textit{arm-independent} delay. And \citet{pike}, who consider an aggregated rewards model.

Delayed payoff was also studied from an adversarial perspective, where both the delay and rewards are adversarial. This includes works such as \cite{cesa2019delay,thune,Bistritz,gyorgy2021adapting}.  \citet{masoudian,masoudian2023improved} presented a ``best-of-both-worlds" algorithm for the delayed setting, which  is a modification of \citet{zimmert}. Interestingly, in the adversarial setting, if the delay is adversarial and arm dependant as in \cite{hoeven}, the adversary can correlate the payoff and the delay, thus the payoff also depends on the delay. While this resembles the payoff-dependant setting, the resulting regret bounds are very different. In particular, the delay has a multiplicative effect on the regret, while in the stochastic case we only suffer an additive term.

Only few works considered regret in the \textit{reward-dependent} setting. \citet{tal} considers the case where the delay and reward are sampled from a joint distribution. 
In their work, there is no assumption on the delay distribution, in particular it may be unbounded, while the reward is bounded in [0,1]. Their regret bound has an additive term proportional to the maximum delay.
Later, \citet{yifu}, consider the setting where the delay equals the reward. Their additive term in the regret is $D\sum_i \Delta_i$. For the same setting we show an additive regret of $\bar{d}$, which we can also improve to $\min(\bar{d},D\max_i\Delta_i)$.

\section{Problem Setup}
\label{sec:problem setup}
Our \emph{delay-as-payoff} model is as follows. There is a set $[K]$ of $K$ arms. Each arm $i\in [K]$ has a distribution $\D_i$ with support $[D]\cup\{0\}$, where $D$ is the maximum delay.
In each step $t = 1, 2,...,T$, the agent chooses arm $i_t\in [K]$, and incurs a delay $d_t\sim \D_{i_t}$. The agent observes the payoff of $i_t$ at time $t+d_t$. The payoff is $d_t/D$,  which we denote by $r_t$ for rewards, or $c_t$ for cost.  Thus, the average payoff is $E_{X\sim \D_i}[X/D]$ denoted by $\mu(i)$. 
Until step $t+d_t$, we refer to the payoff of arm $i_t$ as \emph{missing}, since we do not know its actual delay, and thus payoff, yet. At step $t+d_t$, $d_t$ is revealed, and thus the agent observes the payoff. 
The interaction protocol is  in Algorithm~\ref{protocol1}. 

\begin{algorithm}
\caption{Protocol1}\label{protocol1}
\begin{algorithmic}
\For{$t\in[T]$}
    \State \text{Agent picks an action $i_t\in[K]$}
    \State \text{Environment samples $d_t\sim \D_{i_t}$}
    \State Agent observes feedback 
    $\{d_s:t=s+d_s\}$
\EndFor
\end{algorithmic}
\end{algorithm}

The performance of the agent is measured by the \textit{expected pseudo regret}, which is the difference between the algorithm's cumulative expected payoff and the best expected payoff of any fixed arm.
In the case of reward this will be:
\begin{align*}
    \mathcal{R}_T&= \max_{i\in [K]}T\mu(i)
     - \mathbb{E} \left[\sum_{t=1}^T r_t\right] = \mathbb{E}\left[\sum_{t=1}^T \Delta_{i_t}\right]
\end{align*}

And in the case of cost:
\begin{align*}
    \mathcal{R}_T&=\mathbb{E} \left[\sum_{t=1}^T c_t\right]- \min_{i\in [K]}
    T\mu(i) = \mathbb{E}\left[\sum_{t=1}^T \Delta_{i_t}\right]
\end{align*}
where $\Delta_{i}$ is the sub-optimality gap of arm $i$, i.e., $\Delta_i=|\mu(i)-\mu^*|$ and $\mu^*=\max_{i\in[K]} \mu(i)$, for rewards, and $\mu^*=\min_{i\in[K]} \mu(i)$, for cost. Respectively we define $i^*$ = $i\in[K]\;s.t.\; \mu(i) = \mu^*$. Which, without loss of generality, we assume to be single.
For readability, we make use of the following additional notations; $d(i) = D\mu(i)$ is the mean delay of arm $i\in[K]$ and $\Delta_{max} = \max_{i\in[K]}\Delta_i$.

\section{Delay as cost}
\label{sec:cost}

In this section, we consider the case where the cost is proportional to the agent's delay. We introduce our main algorithm, Bounded Doubling Successive Elimination (\texttt{BDSE}, \Cref{alg:BDSE_mt}), and its associated subroutine, Cost Successive Elimination (\texttt{CSE}, \Cref{alg:CSE_cost_mt}). \texttt{CSE} builds on the well-known Successive Elimination (SE) algorithm \cite{JMLR:v7:evendar06a}, and as discussed in \Cref{sec:cse}, it introduces an improved lower confidence bound (LCB). This LCB leverages not only the observed payoff but also takes into account the number of missing observations and their current duration. As we discuss later in this section, a similar improvement cannot be obtained for an upper confidence bound. 
Instead, \texttt{BDSE} employs a doubling scheme that upper bounds $\mu^*$. This combination is a key component that enables us to achieve our optimal bounds.
In the following subsections, we expand on these algorithms and their regret guarantees.

\subsection{CSE Algorithm}
\label{sec:cse}

Much like standard SE, \texttt{CSE} maintains a set of active arms, where initially all arms are active. The algorithm works in rounds, where in each round each active arm is selected once.
Unlike standard SE, which eliminates an arm only when there is confidence that it is suboptimal, \texttt{CSE} also eliminates an arm when there is confidence that it is worse than a specific threshold parameter $B$.
In our following definitions we distinguish between the three following groups (Note that they are not mutually exclusive):
\begin{enumerate}
    \item $M_t(i)$ are the time steps with chosen arm $i$ that have not returned feedback by time $t$. Formally, $M_t(i)=\sett{s\in[t]}{i_s=i\land s+d_s\geq t})$. We denote the size of this group $m_t(i) = |M_t(i)|$.
    \item $\obs_t(i)$ are time steps with chosen arm $i$ that have returned feedback by time $t$. Formally, $\obs_t(i)=\sett{s\in[t]}{i_s=i\land s+d_s< t}$.
    \item $F_t(i)$ are the time steps with chosen arm $i$ that are at least $D$ time steps ago, hence their feedback must have returned. Formally, $F_t(i)=\sett{s\in[t]}{s \leq t-D}$
    \item Additionally, $n_t(i) = |\{i_s = i|1 \le s \le t\}|$ is the number of all plays of arm $i\in[K]$ before step $t\in[T]$.
\end{enumerate}

Our lower-confidence-bound comprises three terms, $LCB_t(i) = \max\{ L_t^1(i), L_t^2(i), L_t^3(i) \}$; each bounds the expected cost with high probability:
\begin{enumerate}
    \item $L_t^1(i)$ incorporates both observed and unobserved samples, optimistically assuming that the payoff of unobserved samples will be received in the next round. 
    Formally, 
    \begin{align}
    \label{eq:L1}
        L_t^1 & = \hat{\mu}_t^-(i)-\sqrt{\frac{2\log T}{n_t(i)}},
    \end{align}
    where $\hat{\mu}_t^-(i) = \dfrac{1}{n_t(i)}(\sum_{s \in M_t(i)}{\dfrac{t - s}{D}} + \sum_{s \in \obs_t(i)}{c_s})$. 
    \item $L_t^2(i)$ uses only observed samples that were played up to time $t-D$:
    \begin{align}
    \label{eq:L2}
        L_t^2(i) & = \hat{\mu}^F_{t}(i) - \sqrt{\frac{2\log T}{|F_t(i)| \vee 1}}
    \end{align}
    where $\hat{\mu}^F_t(i) = \dfrac{1}{|F_t(i)| \vee 1}(\sum_{s \in F_t(i)}{c_s})$ is the empirical average of those samples ($\vee$ indicates $\max$). 
    We take maximum in the denominator for the case that some arm was not played by $t-D$, this can occur until $t = D+K$.
    \item $L_t^3(i)$ directly leverages the fact the cost corresponds to the magnitude of delay. In particular, as we establish in \Cref{lem:w_upper_bound_mt}, the number of missing samples can't be much larger than a factor of the expected delay. We define,
    \begin{align}
    \label{eq:L3}
        L_t^3(i) & = 
        \dfrac{|S_t|}{D} \roundy{\dfrac{m_t(i)}{2} - 8\log T - 1}
    \end{align}
    where $S_t$ is the set of active arms at time $t$. With the use of \Cref{lem:w_upper_bound_mt}, $L_t^3(i)$ serves as a valid lower-confidence bound for $\mu(i)$.
\end{enumerate}

For the elimination step we require an upper confidence bound. We use a similar bound as in $L_t^2$:
\begin{align}
    UCB_t(i) & = \hat{\mu}^F_{t}(i) + \sqrt{\frac{2\log T}{|F_t(i)| \vee 1}}
    \label{eq:UCB}
\end{align}

The \texttt{CSE} algorithm is formally described in \Cref{alg:CSE_cost_mt} and as a full pseudo code in \Cref{alg:CSE} in the supplementary material.

Note that if the delays were deterministic, then we would have $m_t(i)\leq d(i)/R$, for every arm $i$.
The following lemma handles the case that the delays are stochastic with expectation $d(i)$.

\begin{lemma} \label{lem:w_upper_bound_mt}
For every step $t$, if the last $\min \curly{D, t}$ steps was played with a round robin of a set of size at least $R$: 
\begin{align*}
    Pr\squary{m_t(i) \le \dfrac{2d(i)}{R} + 16\log T + 2} \geq 1 - \dfrac{1}{T^2}
\end{align*}
\end{lemma}

Note that using the missing plays to upper bound the mean delay results in a significantly weaker bound,
and thus unhelpful. With that in hand, and standard concentration bounds, we can define an event $G$ that happens with high probability.
\begin{definition} \label{def:g_mt}
Assume that the actions were played in a round robin manner. Denote $R_t$ to be minimum size of the round robin by time $t\in[T]$. Let $G$ be the event that for every $t\in[T]$ and $i\in[K]$:
\begin{align}
\label{eq:good-missing}
m_t(i) &\le \dfrac{2d(i)}{R_t} + 16\log T + 2
\nonumber
\\
|\mu(i) - \hat{\mu}_t(i) | & \le \sqrt{\dfrac{2\log T}{n_t(i)}}
\end{align}
where, $\hat{\mu}_t(i) = \dfrac{1}{n_t(i)}\sum_{s\in \{1 \le s \le t \mid i_s = i\}}{r_s}$ is the empirical average of payoff of time steps with chosen arm $i$. Note that due to missing plays, this is likely unknown to the algorithm at time $t$.
\end{definition}
We show that $G$ holds with high probability.
\begin{lemma} \label{lem:good_g_prob}
\label{lemma:good-G}
    The event $G$ holds with probability $1-3/T^2$.
\end{lemma}

\begin{algorithm}[t]
    \caption{Cost Successive Elimination (\texttt{CSE})}\label{alg:CSE_cost_mt}
    \begin{algorithmic}
    \State \textbf{Input:} number of rounds $T$, number of arms $K$, maximum delay $D$, Elimination Threshold $B$.
    \State \textbf{Initialization:} $t\gets1$, $S\gets[K]$
    \State \textbf{Output :} Status (either \texttt{Success} or \texttt{Fail}) and $t$ number of time steps performed.
    \While{$t<T$}
    \State Play each arm $i\in S$
    \State Observe any incoming feedback
    \State Set $t\gets t+|S|$
    \For{$i\in S$}
        \State $LCB_t(i) \gets \max\{L^1_t(i), L^2_t(i), L^3_t(i)\}$ \\as defined in \Cref{eq:L1,eq:L2,eq:L3}
        \State Update $UCB_t(i)$ as defined in \Cref{eq:UCB} 
    \EndFor
        \State {\color{gray}$\vartriangleright$ \textit{Elimination step}}
        \State Remove from $S$ any arm $i$ if there exists $j$ such that $\min\curly{UCB_t(j),B} < LCB_t (i)$
        \If{$S = \emptyset$}
            \State Return (\texttt{Fail},$t$)
        \EndIf
    \EndWhile
    \State Return (\texttt{Success},$t$)
    \end{algorithmic}
\end{algorithm}

As previously mentioned, \texttt{CSE} adopts a less conservative elimination rule than standard SE, as it also eliminates arms that perform worse than a specified threshold $B$. Consequently, it might eliminate all arms, in which case it would return a \texttt{Fail}. For this reason we have a main program \texttt{BDSE} that call \texttt{CSE} with the threshold $B$. When  \texttt{CSE} return a \texttt{Fail} back to \texttt{BDSE}, then \texttt{BDSE} doubles the threshold $B$ and calls \texttt{CSE} with the new threshold. 

The following theorem shows that \texttt{CSE} will not eliminate the optimal arm, if $B \ge \mu^*$.

\begin{lemma}[Safe Elimination] 
\label{thm:CSE_cost_safe_elimination_mt}
\label{lemma:CSE_cost_safe_elimination_mt}
    Assuming $G$ holds and $B \ge \mu^*$, the procedure \texttt{CSE} will not fail and $i^*$ will not be eliminated.
\end{lemma}

The following theorem bounds the regret suffered in one call to the procedure \texttt{CSE}.

\begin{theorem}
\label{thm:CSE_bound}
    The regret of \texttt{CSE} (\Cref{alg:CSE_cost_mt}) with elimination threshold $B$ is bounded by,
    \begin{align*}
        \sum_{i:\Delta > 0}{\dfrac{129\log T}{\Delta_i}} + 8D\min\curly{B, D\Delta_{max}}\log K
    \end{align*}
\end{theorem}

The first term in the regret scales as optimal instance-dependent, non-delayed MAB. The second term scales with the magnitude of $B$. Note that, by \Cref{thm:CSE_cost_safe_elimination_mt}, $B$ will remain smaller than $2 \mu^*$ and thus the second term is at most $\tilde O(d^*)$. 

\begin{proofsketch}
For the sake of simplicity we provide the proof sketch only for the $B$ term in the \textit{min}. Assume the good event $G$ holds.  Let $S_t$ be the set $S$ at time $t$. Fix any sub-optimal arm $i\in [K]$, and let $\tau_i$ be the last elimination step which arm $i$ remained active.

Recall that $L_t^1$ is computed with an optimistic empirical average $\hat\mu_t^-(i)$. That is, any missing sample is assumed to be observed in the next round. At worse, such missing sample eventually would return after $D$ steps and would have cost of $1$. Thus, the difference between $\mu_t^-(i)$ and the actual empirical mean $\hat\mu_t(i)$ is at most ${m_t(i)}/{n_{t}}$. Since the good event $G$ holds, and the arm $i$ was not yet eliminated at time $\tau_i$,
\begin{align}
    \notag \mu(i) &\le LCB_{\tau_i}(i) + 2\sqrt{\frac{2\log T}{n_{\tau_i}(i)}} + \dfrac{m_{\tau_i}(i)}{n_{\tau_i}(i)} \\
    \notag & \le B + 2\sqrt{\frac{2\log T}{n_{\tau_i}(i)}} + \dfrac{m_{\tau_i}(i)}{n_{\tau_i}(i)}\\
    \notag & \approx B + 2\sqrt{\frac{2\log T}{n_{\tau_i}(i)}} + \dfrac{2 D L^3_{\tau_i}(i)}{n_{\tau_i}(i) |S_{\tau_i}|} \\
    &\le B + 2\sqrt{\frac{2\log T}{n_{\tau_i}(i)}} + \dfrac{2DB}{n_{\tau_i}(i)|S_{\tau_i}|} \label{eq:cost_mu_r_mt}
\end{align}
We consider three cases: (i) $\mu^* < B$ and $\mu(i) < 2B$, (ii) $B \le \mu^*$ and (iii) $2B \le \mu(i)$.

\paragraph{case (i):} 
By \Cref{thm:CSE_cost_safe_elimination_mt}, we know that $i^*$ will not be eliminated (since $\mu^* < B$).
Since $i$ was not eliminated at time $\tau_i$, $L_{\tau_i}^2(i) \leq UCB_{\tau_i}(i)$. Using standard arguments this implies that $\Delta_i n_{t_{\tau_i}}^{F}(i) \le O({{\log (T)}/{\Delta_i}})$. Recall that $n_{t_{\tau_i}}^{F}(i)$ is the number of times we played $i$ until time $\tau_i - D$. In the last $D$ plays $i$ was played approximately ${D}/{|S_{\tau_i}|}$ times due to the round-robin, which causes additional regret of $\Delta_i\frac{D}{|S_{\tau_i}|} \le O\big( {\frac{DB}{|S_{\tau_i}|}} \big)$. This accumulates to a total regret of $O\big( {\frac{\log T}{\Delta_i} + \frac{DB}{|S_{\tau_i}|}} \big)$. 

\paragraph{case (ii):}  We use \Cref{eq:cost_mu_r_mt} to show that $\Delta_i \le 2\sqrt{\frac{2\log T}{n_{\tau_i}(i)}} + \frac{2DB}{n_{\tau_i}(i) |S_{\tau_i}|}$. This implies that either $\Delta_i \le 4\sqrt{\frac{2\log T}{n_{\tau_i}(i)}}$ or $\Delta_i \le \frac{4DB}{n_{\tau_i}(i) |S_{\tau_i}|}$. In the first case we get that the regret from arm $i$ is bounded by $O({\log (T)}/{\Delta_i})$. Similarly, in the second case the regret is bounded by  $O({(DB)}/{|S_{\tau_i}|})$. 

\paragraph{case (iii):} 
The third case assumes that $B\leq \mu(i)/2$. By rearranging the terms of \Cref{eq:cost_mu_r_mt}, we have that $\Delta_i \le \mu(i) \le O\Big( {\sqrt{\frac{\log T}{n_{\tau_i}(i)}} + \dfrac{DB}{n|S_{\tau_i}|}} \Big)$. Similarly to case (ii), the total regret is bounded by $O\big({\frac{\log T}{\Delta_i} + \frac{DB}{|S_{\tau_i}|}} \big)$.

Now we can sum over all sub-optimal arms $i$ and bound the regret by $O\big({\sum_{i:\Delta > 0}{\frac{\log T}{\Delta_i}} + DB\log K}\big)$. Note that the regret when $G$ does not hold is in expectation only $O(1)$.
\end{proofsketch}

\subsection{Bounded Doubling Successive Elimination}
\texttt{CSE} (\Cref{alg:CSE_cost_mt}) demands a parameters $B$, which is not available for the agent. In this algorithm we estimate $B$ using the "doubling" technique.
\begin{algorithm}[t]
\caption{Bounded Doubling Successive Elimination}\label{alg:BDSE_mt}
\begin{algorithmic}
\State \textbf{Input:} number of rounds $T$, number of arms $K$, maximum delay $D$.
\State \textbf{Initialization:} $B\gets{1}/{D}$
\While{$t<T$}
\State Run $(\texttt{ret},\tau)\gets \texttt{CSE}(T-t, K, D, B)$
\If{\texttt{ret}=\texttt{Fail}}
    \State $B \gets 2  B$
    \State $t \gets t + \tau $
\EndIf
\EndWhile
\end{algorithmic}
\end{algorithm}
\begin{corollary}
\label{corr:bdse_regret}
Algorithm \texttt{BDSE} has a regret of at most,
\begin{align*} 
    \sum_{i:\Delta > 0}{\dfrac{129\log T\log d^*}{\Delta_i}} + 8\min\curly{d^*, D\Delta_{max}\log d^*}\log K
\end{align*}
\end{corollary}

\begin{proof}
From \Cref{thm:CSE_cost_safe_elimination_mt} we know that if $B \ge \mu^*$ then \texttt{CSE} will not fail, which means that the number of calls to \texttt{CSE} is at most $\log d^*$. Notice that on the $j$'s call of $BSE$, $DB = 2^j$. The total regret will be at most
\begin{align*}
     &\sum_{j=0}^{\log d^*}{\roundy{\sum_{i:\Delta > 0}{\dfrac{129\log T}{\Delta_i}} + 8\cdot \min\curly{2^j, D\Delta_{max}}\log K}} \\
    &= \sum_{i:\Delta > 0}{\dfrac{129\log T\log d^*}{\Delta_i}} + 8\min\curly{d^*, D\Delta_{max}\log d^*}\log K 
\end{align*}
\end{proof}

\subsection{Lower Bound}
\label{sec:clb}
In this section we show two lower bounds for the cost setting. The first is a general lower bound which nearly matches the regret bound of our algorithm. And the second, a lower bound for classical SE algorithms. The main challenge is to understand the impact of $d^*$ on the regret. We focus on the second term of the upper bound as the first term, $\sum_{i:\Delta > 0}\frac{\log T}{\Delta_i}$, is a well known instance dependent bound even when there are no delays \cite{bubeck2012regret}.

\begin{theorem}\label{thm:cost_lb_mt}
In the cost scenario, for every choice of $d^* \le {D}/{2}$, there is an instance for which any algorithm will have a regret of $\Omega\roundy{d^*}$
\end{theorem}
\begin{proof}
We consider two arms with deterministic delays, one is $d^*$ (and cost $\mu^*=d^*/D\leq 1/2$) and the other is $D$ (with cost $\mu=1$).
We select at random which arm has delay $d^*$ and which $D$. Until time $d^*$ both arms are indistinguishable, and hence the regret is $(1-\mu^*)d^*$. Since $\mu^*\leq 1/2$ we have a regret of at least $d^*/2$.
\end{proof}

\paragraph{Conservative SE algorithms:}
We show, for a natural class of SE algorithms, which are also conservative (w.h.p. do not eliminate the optimal action), a lower bound of $\sqrt{D d^*}$. Interestingly, this bound is also tight, as we show in \Cref{sec:regular successive elimination} a conservative SE algorithm which attains it.

For this impossibility result we use the following two problem instances.
In the first problem instance we have the delay of arm $1$ to be $\sqrt{Dd^*}/2$ w.p. $2\sqrt{d^*/D}$ and otherwise $0$. For arm $2$ the delay is deterministic $D$. 
In the second problem instance we have the delay of arm $1$ to be $D$ w.p. $2\sqrt{d^*/D}$ and otherwise $0$. For arm $2$ the delay is deterministic $\sqrt{Dd^*}$. In the first instance the best arm is $1$ while in the  second it is arm $2$. Until time $\sqrt{Dd^*}$ we cannot distinguish between the two instances, so a conservative SE algorithm will keep playing both arms in a round-robin manner, and have a regret of $\Omega(\sqrt{Dd^*})$ in the first instance.

It is worth observing why our BDSE overcomes those two problem instances. Due to the doubling scheme, every time the number of missing plays reaches (roughly) the current threshold both arms will be eliminated, until the threshold surpasses $\mu^*$. In the first instance this will happen after $d^*$ steps, at which point only the optimal arm will remain, and thus the regret is at most $d^*\Delta = d^*$. Similarly in the second instance, the threshold will surpass $\mu^*$ after $\sqrt{Dd^*}$ steps. The $\Delta$ here is $\sqrt{d^*/D}$ and so the regret is at most $d^*$.

\section{ Delay as Reward}
\label{sec:reward}
In this section we consider the case where the delay corresponds to the agent's reward. Similarly to cost, we have a main program Bounded Halving Successive Elimination algorithm (BHSE, \Cref{alg:BHSE_mt}), and an associated subroutine Reward Successive Elimination (RSE). Besides the transition from minimization to maximization, the main difference is that the missing feedbacks at time $t$ should be interpreted differently. In the following subsections, we include the details of these algorithms and their regret guarantees.

\subsection{RSE Algorithm}
As in the cost scenario, we start with a Reward Successive Elimination algorithm. Since we consider rewards, we would like the threshold $B$ to decrease with time (rather than increase, as was done in the cost scenario).
Eventually, RSE expects  $B\leq \mu^*$ to guarantee success.
As in the \texttt{CSE} algorithm, we will eliminate arms based on suboptimality, in comparison to other arms, or when there is confidence that they are worse than the parameter $B$.

Our upper-confidence-bound comprises of only two terms, $UCB_t(i)=\min\{U_t^1(i), U_t^2(i)\}$; which are analogues to $L_1$ and $L_2$ in the cost case. Formally,
    \begin{align}
    \label{eq:U1}
        U_t^1 & = \hat{\mu}_t^+(i) + \sqrt{\frac{2\log T}{n_t(i)}},
    \end{align}
    where $\hat{\mu}_t^+(i) = \dfrac{1}{n_t(i)}(\sum_{s \in M_t(i)}{1} + \sum_{s \in \obs_t(i)}{r_s})$ is an optimistic estimate of $\mu(i)$. (Recall that $r_s=d_s/D$.)
    Similarly, $U_2$ as well as $LCB_t$ are defined by,
    \begin{align}
    \label{eq:U2}
        U_t^2(i) & = \hat{\mu}_t^{F}(i)+\sqrt{\frac{2\log T}{|F_t(i)|\vee 1}},
    \end{align}
    For the LCB we have,
        \begin{align}
        LCB_t(i) & = \hat{\mu}_t^{F}(i) - \sqrt{\frac{2\log T}{|F_t(i)|\vee 1}}
        \label{eq:LCB}
    \end{align}
    where $\hat{\mu}_t^{F}(i)= \dfrac{1}{|F_t(i)| \vee 1}(\sum_{s\in F_t(i)}{r_s})$.

\begin{algorithm}[t]
\caption{Reward Successive Elimination}\label{alg:RSE_reward_mt}
\begin{algorithmic}
    \State \textbf{Input:} number of rounds $T$, number of arms $K$, maximum delay $D$, Elimination Threshold $B$.
    \State \textbf{Initialization:} $t\gets1$, $S\gets[K]$
    \State \textbf{Output :} Status (either \texttt{Success} or \texttt{Fail}) and $t$ number of time steps performed.
    \While{$t<T$}
    \State Play each arm $i\in S$
    \State Observe incoming payoff from $\{s:s+d_s=t\}$
    \State Set $t\gets t+|S|$
    \For{$i\in S$}
        \State $UCB_t(i) \gets \min\{U^1_t(i), U^2_t(i)\}$ \\as defined in \Cref{eq:U1,eq:U2}
        \State Update $LCB_t(i)$ as defined in \Cref{eq:LCB} 
    \EndFor
        \State {\color{gray}$\vartriangleright$ \textit{Elimination step}}
        \State Remove from $S$ any arm $i$ if there exists $j$ such that $\max\curly{LCB_t(j),B} > UCB_t (i)$
        \If{$S = \emptyset$}
            \State Return (\texttt{Fail}, $t$)
        \EndIf
    \EndWhile
    \State Return (\texttt{Success}, $t$)
\end{algorithmic}
\end{algorithm}

We use the same good event $G$ as defined in \Cref{def:g_mt} in the previous section which also holds here w.h.p.

\begin{theorem}[Safe Elimination] \label{thm:bse_reward_safe_elimination_mt}
\label{lem:bse_reward_safe_elimination_mt}
Assuming $G$ holds and $B \le \mu^*$, the procedure \texttt{RSE} will not return \texttt{Fail} and $i^*$ will not be eliminated.
\end{theorem}

The following theorem bounds \texttt{RSE}'s regret.
\begin{theorem} \label{thm:rse_regret}
Assume $B \ge \frac{\mu^*}{2}$. The regret of \texttt{RSE} (\Cref{alg:RSE_reward_mt}) with elimination threshold $B$ is bounded by,
\begin{align*}
    \sum_{i:\Delta > 0}{\dfrac{289\log T}{\Delta_i}} + 12\min\curly{\bar{d}, D\Delta_{max}}\log K,
\end{align*}
where $\bar{d}$ is the second highest expected delay.
\end{theorem}
The assumption that $B \geq \mu^* / 2$ is satisfied under the main program \texttt{BHSE} (\Cref{alg:BHSE_mt}) due to \Cref{lem:bse_reward_safe_elimination_mt}.

\subsection{Bounded Halving Successive Elimination} 
\label{sec:bhse_mt}
Similar to the main program in the cost case, $\texttt{BHSE}$ estimates a lower bound for $\mu^*$. It starts with an over-estimation of $B=1$ and this time halves it by $2$ whenever $\texttt{RSE}$ returns $\texttt{Fail}$.

\begin{algorithm}
\caption{Bounded Halving Successive Elimination (\texttt{BHSE})}
\label{alg:BHSE_mt}
\begin{algorithmic}
\State \textbf{Input:} number of rounds $T$, number of arms $K$, maximum delay $D$.
\State \textbf{Initialization:} $B\gets 1$
\While{$t<T$}
\State Run $(\texttt{ret}, \tau) = \texttt{RSE}(T-t, K, D, B)$
\If{\texttt{ret} = \texttt{Fail}}
    \State $B \gets B / 2$
    \State $t \gets t + \tau$   
\EndIf
\EndWhile
\end{algorithmic}
\end{algorithm}

\begin{corollary} \label{corr:BHSE_regret}
Algorithm \texttt{BHSE} has regret of at most,
\begin{align*}
    \roundy{\sum_{i:\Delta > 0}\dfrac{289\log T}{\Delta_i} + 12\min\curly{\bar{d}, D\Delta_{max}}\log K}\log \dfrac{1}{\mu^*}
\end{align*}
\end{corollary}
\begin{proof}
From \Cref{thm:bse_reward_safe_elimination_mt} we know that if $B \le \mu^*$ \texttt{BSE} will not fail, which means that the loop will run a maximum of $\log ({1}/{\mu^*})$ times. This also means that $B \ge {\mu^*} / {2}$, as needed. 
Therefore, the total regret will be $\roundy{\sum_{i:\Delta > 0}\dfrac{289\log T}{\Delta_i} + 12\min\{\bar{d}, D\Delta_{max}\}\log K}\log \dfrac{1}{\mu^*}$
\end{proof}

Note that without loss of generality we can assume that $\mu^* \geq 1/T$, since otherwise $\Delta_i \leq 1/T$ for all arms and the regret is trivially bounded by $1$. Therefore the term $\log ({1}/{\mu^*})$ would be at most $\log T$.

Notice that unlike \Cref{corr:bdse_regret}, the regret bound depends on $\bar{d}$. On the one hand, this is better than the delay of the best arm (which has the maximal expected delay). On the other hand, this can be much larger than the regret in the cost scenario, which depends on the minimal expected delay. This matches the lower bound in \Cref{sec:rlb}.

\subsection{Lower Bound}
\label{sec:rlb}

In this section we show a lower bound for the reward setting, which nearly matches our regret bound. As mentioned before, we will focus on understanding the impact of $\bar d$ since $\sum_{i:\Delta > 0}\frac{\log T}{\Delta_i}$, is a well known instance dependent bound even when there are no delays. 

\begin{theorem} \label{thm:reward_lb_mt}
In the reward scenario, for every choice of $\bar{d} \le {D}/{2}$, there is an instance for which any algorithm will have a regret of $\Omega\roundy{\bar{d}}$
\end{theorem}

\begin{proof}
Consider the case of $K$ arms. We have one arm with constant delay $D$, one arm with constant delay $\bar{d}=D/2$, and the remaining arms have delay $0$ (and hence reward $0$). We select at random the identities of the arms. This implies that for any sub-optimal arm $i$ we have $\Delta_i\geq 1/2$.
Clearly,
until time $\bar d$ the best two arms are indistinguishable hence the regret is at least of order of $\bar d \min_i\Delta_i \geq \bar d / 2$.
\end{proof}
\paragraph{Conservative SE algorithms:}
Using similar arguments as in the cost case we can show here a lower bound of $\sqrt{D\bar{d}}$.

\section{Experiments}
\label{sec:experiments}

We conducted synthetic experiments for both the \textit{cost} and the \textit{reward} setting, measuring against the algorithms mentioned in \Cref{table:1}. For both settings we show results on two representative distributions: Truncated Normal (bounded in $[0, D]$) and Bernoulli.
(Due to space constraints, we defer to the supplementary material the experiments regrading cost.)

All experiments are run with parameters: $T=150,000$, $K=30$ and $D=5000$. For the truncated Normal we sample $K$ means and standard deviations, and adjust them to get a truncated version. Since our additive term in the regret is $\min\{\bar{d},D\Delta_{max}\}$, our contribution is mainly for instances where $\bar{d}<D\Delta_{max}$. Hence, we show the result on such instance, by using an exponential distribution to sample the means of the arms, resulting in sparsity in the higher regime. The standard deviations are sampled uniformly in $[0,D]$. We then computed analytically the expected mean of each arm, to compute the regret. For the Bernoulli distribution, we sample $K$ probabilities $p_i$ uniformly in $[0,1]$, so that arm $i$ gets $0$ with probability $p_i$ and $D$ with probability $1-p_i$.
\Cref{fig:exp-reward} shows the average cumulative regret, averaged on 10 runs. The shaded region is the standard deviation of the 10 runs. \texttt{BHSE} outperforms \texttt{OPSE} and \texttt{CensoredUCB} in both distributions. It is also evident that Bernoulli distribution is more challenging, resulting in both higher regret, and larger standard deviation. 
In the cost setting,  \texttt{BDSE} outperforms \texttt{OPSE}. The high standard deviation of \texttt{BDSE} in the Bernoulli case, occurs because a small variation in the means can lead to a different number of doubling calls. In practice, this can easily be smoothed by not ``throwing" the history every doubling call.

\begin{figure}
    \centering
    \includegraphics[width=0.82\linewidth]
    {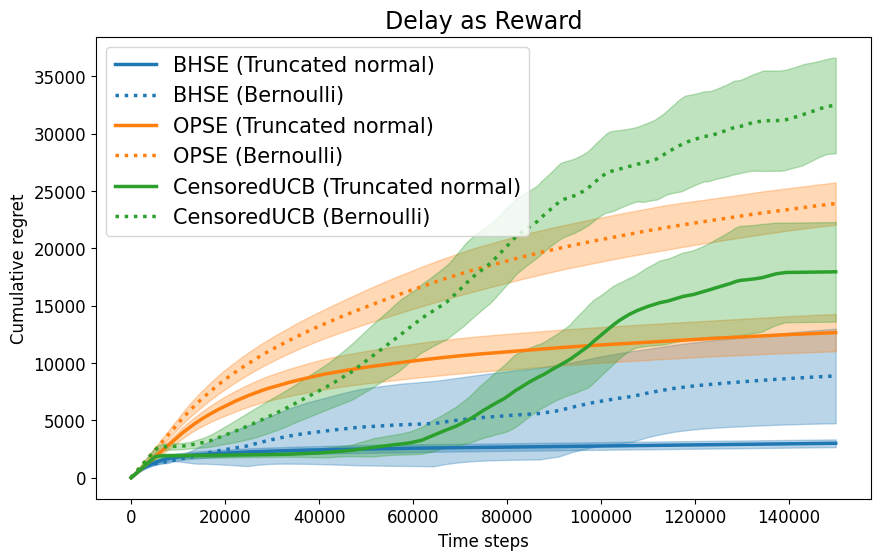}
    \includegraphics[width=0.82\linewidth]{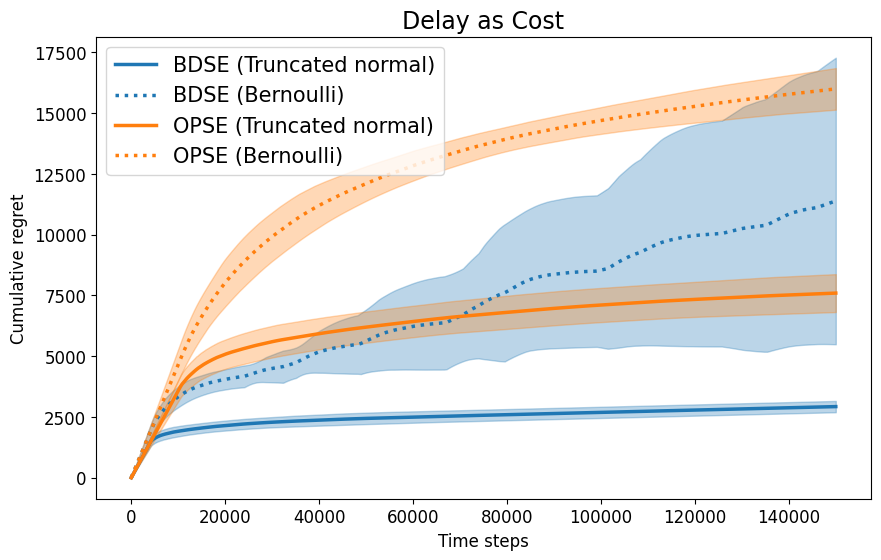}
    \caption{This graph shows results of experiments on different algorithms (color) and different distributions (line style).}
    \label{fig:exp-reward}
\end{figure}

\section{Discussion}
\label{sec:discussion}
In this paper, we explored a variant of the classical MAB problem, where the payoff is both delayed and directly corresponds to the magnitude of the delay. For the delay as reward setting we provided a detailed analysis of this problem and introduced tighter upper and lower regret bounds compared to those established in previous works. We are the first to generalize also to cost, highlighting the inherent difference between cost and reward in this setting.

There are several interesting future directions. 
First, as our motivation for the reward setting is the online advertising, it is a natural question to ask if we can expect similar results in the a \textit{delay as payoff} contextual bandit setting or other variants of the MAB problem (e.g., linear contextual bandit). 
Furthermore, it remains unclear whether our results can be generalized for more general delay distributions which are potentially unbounded (but have a bounded expectation).
Finally, by taking a similar perspective in the adversarial setting, the assumption that the delays correspond to the payoffs is a challenging new problem.

\section{Acknowledgments}
We thank Batya Berzack and Yoav Nagel for their input and work during the initiation of this paper and Ran Darshan for his constructive feedback.
OS is supported by the Darshan lab.
IC is supported by the Israeli Science Foundation (grant number 1186/18).
TL and YM are supported by the European Research Council (ERC) under the European Union’s Horizon 2020 research and innovation program (grant agreement No. 882396), by the Israel Science Foundation and the Yandex Initiative for Machine Learning at Tel Aviv University.
IC, TL and YM are also supported by a grant from the Tel Aviv University Center for AI and Data Science (TAD).

\newpage

\bibliographystyle{plainnat}
\bibliography{cited.bib}

\newpage

\appendix

\section{Summary of notations}
\label{apndx: notations}
For convenience, the table below summarizes most of the notation used throughout the paper.
\begin{table}[ht]
    \centering
    \begin{tabular}{c|l}
    $T$ & Number of rounds
    \\
    $K$ & Number of actions
    \\
    $\mathcal{D}_i$ & Delay distribution of arm $i$
    \\
    $D$ & Maximum possible delay
    \\
    $i_t$ & The agent's action at time $t$
    \\
    $\mu(i)$ & The expected payoff of arm $i$: $\mu(i) = E_{X\sim \D_i}[X/D]$
    \\
    $d(i)$ & The expected delay of arm $i$: $d(i) = D\mu(i)$
    \\
    $d_t$ & Realized delay at time $t$
    \\
    $r_t$ & Realized reward at time $t$ for the reward setting. I.e., $r_t = d_t/D$
    \\
    $c_t$ & Realized cost at time $t$ for the cost setting. I.e., $c_t = d_t/D$
    \\
    $\mu^*$ & The expected payoff of the optimal arm. I.e., 
    $\mu^*=\max_{i\in[K]} \mu(i)$, for rewards, and $\mu^*=\min_{i\in[K]} \mu(i)$
    \\
    $i^*$ & The optimal arm. I.e., 
    $i^*=\arg\max_{i\in[K]} \mu(i)$, for rewards, and $i^*=\arg\min_{i\in[K]} \mu(i)$
    \\
     $\Delta_{i}$ & The suboptimality gap of arm $i$, i.e., $\Delta_i=|\mu(i)-\mu^*|$  
     \\
     $S_t$ & Set of active arms at time $t\in[T]$
     \\
     $\tau_i$ & The last elimination point that $i\in[K]$ was not eliminated at 
     \\
     $K_i$ & $|S_{\tau_i}|$, the size of $S$ when $i\in[K]$ was eliminated
     \\
     $\hat{\mu}_t(i)$ & $\frac{1}{n_t(i)}\sum_{s:i_s=i}^t{c_s}$ (or with $r_t$ for reward), the empirical cost/reward (may be unknown to the algorithm)
    \end{tabular}

    \label{tab:notation}
\end{table}

\section{General Lemmas}
\begin{lemma}[\Cref{lem:w_upper_bound_mt} in the main paper] \label{lem:w_upper_bound}
For every step $t$, if the last $\min \curly{D, t}$ steps were played with a round robin of a set of size at least $R$: 
\begin{align*}
    Pr\squary{m_t(i) \le \dfrac{2d(i)}{R} + 16\log T + 2} < 1 - \dfrac{1}{T^2}
\end{align*}
\end{lemma}

\begin{proof}
Let us bound the expectation of $m_t(i)$,
\begin{align*}
    \mathbb{E}[m_{t}(i)] & =\sum_{s=1}^{t}\mathbb{E}\left[\mathbb{I}\{s+d_{s}\geq t,i_{s}=i\}\right]\\
                        & =\sum_{s=1}^{t} \mathbb{I}\{i_{s}=i\} 
                            Pr[d_{s} \geq t-s\mid i_{s}=i]
\end{align*}

Recall that $[d_{s} \mid i_{s}=i] \overset{i.i.d}{\sim} D_i$. So let $d_0(i) \sim D_i$. We have that $Pr[d_{s} \geq t-s\mid i_{s}=i] = Pr[d_{0}(i)\geq t-s]$. Thus,
\begin{align*}
    \mathbb{E}[m_{t}(i)] 
                        & =\sum_{s=1}^{t} \mathbb{I}\{i_{s}=i\} 
                            Pr[d_{0}(i)\geq t-s]\\
                        & =\sum_{s=1}^{t} \sum_{l=1}^{D} 
                            \mathbb{I}\{i_{s}=i,l\geq t-s\} Pr[d_{0}(i)=l]\\
                        & =\sum_{l=1}^{D}Pr[d_{0}(i)=l] \sum_{s=1}^{t} 
                            \mathbb{I}\{i_{s}=i,l\geq t-s\}\\
                        & =\sum_{l=1}^{D}Pr[d_{0}(i)=l] \sum_{s=1}^{t}
                            \mathbb{I} \{i_{s}=i,s\geq t-l\}\\
                        & =\sum_{l=1}^{D}Pr[d_{0}(i)=l] \sum_{s=\max\{1,t-l\}}^{t} 
                            \mathbb{I}\{i_{s}=i\}
\end{align*}

Since between time $\max\{1,t-l\}$ and time $t$ we played round-robin over a set of size at least $R$, 
$\sum_{s=t-l}^{t} \mathbb{I}\{i_{s}=i\} 
\leq \lceil {l}/{R} \rceil
\leq {l}/{R} + 1$. We get that,

\begin{align*}
    \mathbb{E}[m_{t}(i)] 
                        & \leq \sum_{l=1}^{D} Pr[d_{0}(i)=l] 
                            \left(\frac{l}{R} + 1 \right)\\
                        & \leq \frac{1}{R} \sum_{l=1}^{D}l 
                            Pr[d_{0}(i)=l]+1\\
                        & = \frac{d(i)}{R} + 1.
\end{align*}

From \Cref{lem:cons-freedman} using $\delta = \dfrac{1}{T^4}$, we can say that $Pr\squary{m_t(i) \le \dfrac{2d(i)}{R} + 16\log T + 2} \geq 1 - \dfrac{1}{T^4}$. Applying union bound over all arms $i\in[K]$ and all $t\in[T]$ completes the proof. 
\end{proof}

\begin{lemma}[Helper for \Cref{lem:w_lower_bound}]\label{lem:w_lower_bound_helper}
Let $\alpha_1,\alpha_2...\alpha_n$ and $\beta_1,\beta_2...\beta_n$ such that $\alpha_i\in\mathbb{R},\,\beta_i\in\mathbb{R}_{\ge 0}$ 
for every $i\in[n]$, and $\alpha_1\le\alpha_2\le...\alpha_n$. Then for every $k\in[n]$:
\begin{align*}
    \sum_{i=1}^k{\beta_i} \ge \frac{\sum_{i=1}^n{\beta_i}\sum_{j=1}^k{\beta_j\alpha_j}}{\sum_{i=1}^n{\beta_i\alpha_i}}
\end{align*}
\end{lemma}
\begin{proof}
\begin{align*}
\sum_{i=1}^k{\beta_i}\sum_{j=1}^n{\beta_j\alpha_j} &= \sum_{i=1}^k{\beta_i}\sum_{j=1}^k{\beta_j\alpha_j} + \sum_{i=1}^k{\beta_i}\sum_{j=k+1}^n{\beta_j\alpha_j} \\
&\ge \sum_{i=1}^k{\beta_i}\sum_{j=1}^k{\beta_j\alpha_j} + \sum_{i=1}^k{\beta_i}\alpha_{k+1}\sum_{j=k+1}^n{\beta_j} \\
&\ge \sum_{i=1}^k{\beta_i}\sum_{j=1}^k{\beta_j\alpha_j} + \sum_{i=1}^k{\beta_i\alpha_i}\sum_{j=k+1}^n{\beta_j} \\
&= \sum_{i=1}^n{\beta_i}\sum_{j=1}^k{\beta_j\alpha_j}
\end{align*}
\end{proof}

\begin{lemma} \label{lem:w_lower_bound}
Assume that the arms are played in a round robin manner, such that the set size can only decrease as time passes. Then for every $t\in[T]$ and $i\in[K]$:
\begin{align*}
    Pr\squary{m_t(i) \ge \frac{(n_t(i) - n_{t-D}(i))\mu(i) - 1}{2} - 4\log T} \ge 1 - \frac{1}{T^2}
\end{align*}

\end{lemma}
\begin{proof}
Denote $n = n_t(i) - n_{t-D}(i)$, namely the pulls of $i$ in the last $\min\curly{t, D}$ steps. Denote the steps in which the pulls happened as $t-t_1,t-t_2,...t-t_n$. 
\begin{align*}
    \E\squary{m_t(i)} = \sum_{j=1}^n{Pr[d_{t-t_j} \ge t_j]}
\end{align*}
Since those pulls were all from arm $i$, $Pr[d_{t-t_j} \ge t_j] = Pr[d_0(i) \ge t_j]$ for $d_0(i) \sim D_i$.

Denote $t_0 = 0$. Since $Pr[d_0(i) \ge 0] = 1$:
\begin{align*}
    \E\squary{m_t(i)} + 1 &= Pr[d_0(i) \ge 0] + \sum_{j=1}^n{Pr[d_0(i) \ge t_j]} \\
    &= \sum_{j=0}^n{Pr[d_0(i) \ge t_j]} \\
    &= \sum_{j=0}^n \sum_{s=t_j}^{D}{Pr[d_0(i) = s]} \\
    &= \sum_{j=1}^n \sum_{s=t_{j-1}}^{t_j}{j Pr[d_0(i) = s]}  + \sum_{s=t_n}^D{nPr[d_0(i) = s]}
\end{align*}

Denote $R_1,R_2,...R_k$ as the round robin set's sizes, and $n_1,n_2,..n_k$ as the amount of rounds for each set size. Notice that:
\begin{align*}
    t_n &= \sum_{r=1}^k{R_rn_r} \\
    n &= \sum_{r=1}^k{n_r}
\end{align*}
Fix some $j$. Notice that for some $l$ it holds that:
\begin{align*}
    t_j &= \sum_{r=1}^l{R_rn_r} \\
    j &= \sum_{r=1}^l{n_r}
\end{align*}
Since the round robin set size only decreases, by applying  \Cref{lem:w_lower_bound_helper} with $\alpha_r = R_r$ and $\beta_r = n_r$ for every $r\in[n]$,  we get $j \ge \frac{nt_j}{t_n}$. Thus:
\begin{align*}
    \E\squary{m_t(i)} + 1 &\ge \sum_{j=1}^n \sum_{s=t_{j-1}}^{t_j}{\frac{nt_j}{t_n} Pr[d_0(i) = s]}  + \sum_{s=t_n}^D{nPr[d_0(i) = s]} \\
    &= \frac{n}{t_n}\roundy{\sum_{j=1}^n \sum_{s=t_{j-1}}^{t_j}{t_jPr[d_0(i) = s]}  + \sum_{s=t_n}^D{t_nPr[d_0(i) = s]}} \\
    &\ge \frac{n}{t_n}\roundy{\sum_{s=0}^{t_n}{sPr[d_0(i) = s]} + \sum_{s=t_n}^D{t_nPr[d_0(i) = s]}} \\
    &= \frac{n}{t_n}\roundy{\E\squary{d_0(i) \mid d_0(i) \le t_n}Pr\squary{d_0(i) \le t_n} + t_nPr\squary{d_0(i) \ge t_n}} \\
    &= \frac{n}{t_n}\roundy{d(i) - \E\squary{d_0(i) \mid d_0(i) \ge t_n}Pr\squary{d_0(i) \ge t_n} + t_nPr\squary{d_0(i) \ge t_n}}\\
    &= \frac{n}{t_n}\roundy{d(i) + \roundy{t_n - \E\squary{d_0(i) \mid d_0(i) \ge t_n}}Pr\squary{d_0(i) \ge t_n}}
\end{align*}
Notice that:
\begin{align*}
    d(i) &= \E\squary{d_0(i) \mid d_0(i) \ge t_n}Pr\squary{d_0(i) \ge t_n} + \E\squary{d_0(i) \mid d_0(i) \le t_n}Pr\squary{d_0(i) \le t_n} \\
    &\ge \E\squary{d_0(i) \mid d_0(i) \ge t_n}Pr\squary{d_0(i) \ge t_n} \\
    Pr\squary{d_0(i) \ge t_n} &\le \frac{d(i)}{\E\squary{d_0(i) \mid d_0(i) \ge t_n}}
\end{align*}
Since $t_n - \E\squary{d_0(i) \mid d_0(i) \ge t_n} \le 0$:
\begin{align*}
    \E\squary{m_t(i)} + 1 &\ge \frac{n}{t_n}\roundy{d(i) + \roundy{t_n - \E\squary{d_0(i) \mid d_0(i) \ge t_n}}\frac{d(i)}{\E\squary{d_0(i) \mid d_0(i) \ge t_n}}} \\
    &= \frac{n}{t_n}\frac{t_nd(i)}{\E\squary{d_0(i) \mid d_0(i) \ge t_n}}\\
    &\ge \frac{n}{t_n}\frac{t_nd(i)}{D}\\
    &= \frac{nd(i)}{D} \\
    &= n\mu(i)
\end{align*}
Using \Cref{lem:dann} with $\delta = \frac{1}{T^4}$ we can say  that $Pr\squary{m_t(i) \ge \frac{n\mu(i) - 1}{2} - 4\log T} \ge 1 - \frac{1}{T^4}$. Applying union bound over all arm $i\in[K]$ and all $t\in[T]$ completes the proof.
\end{proof}

\begin{lemma} \label{lem:ending}
Let $\A$ be an algorithm, and let $G$ be an event that holds w.p $\epsilon \le \frac{1}{T}$. If for some $\alpha,\beta > 0$, $\A$ guarantees that under $G$ the following holds for every arm $i\ne i^*$:
\begin{align*}
    \Delta_i n_{\tau_i}(i) \le \dfrac{\alpha}{\Delta_i} + \dfrac{\beta}{K_i}
\end{align*}

Then the total regret of $\mathcal{A}$ is bounded by,
\begin{align*}
    \mathcal{R}_T \le \sum_{i:\Delta > 0}{\dfrac{\alpha+2}{\Delta_i}} + \beta\log K
\end{align*}
\end{lemma}
\begin{proof}
Every arm $i \ne i^*$ wasn't played more then once since $\tau_i$, Therefore:
\begin{align*}
    \Delta_i n_{T}(i) \le \dfrac{\alpha}{\Delta_i} + \dfrac{\beta}{K_i} + 1
\end{align*}
Thus:
\begin{align*}
\mathcal{R}_T(\A \;|\; G\; happens)  &= \sum_{i:\Delta > 0}{\Delta_i n_T(i)} \\
&\le \sum_{i:\Delta > 0}{\dfrac{\alpha}{\Delta_i} + \dfrac{\beta}{K_i}} + 1 \\ 
&\le \sum_{i:\Delta > 0}{\dfrac{\alpha}{\Delta_i}} + \beta\log K + K
\end{align*}
In the case $G$ doesn't happen we can say that we pay maximum regret of $T$, which means:
\begin{align*}
    \mathcal{R}_T(\A) &\le \sum_{i:\Delta > 0}{\dfrac{\alpha}{\Delta_i}} + \beta\log K + \epsilon T + K \\
    &\le \sum_{i:\Delta > 0}{\dfrac{\alpha}{\Delta_i}} + \beta\log K + 1 + K \\
    &\le \sum_{i:\Delta > 0}{\dfrac{\alpha+2}{\Delta_i}} + \beta\log K \\
\end{align*}

\end{proof}

\section{Cost equals Delay}
\subsection{Cost Successive Elimination}
\begin{algorithm}
\caption{Cost Successive Elimination (\texttt{CSE})}\label{alg:CSE}
\begin{algorithmic}
\State \textbf{Input:} number of rounds $T$, number of arms $K$, maximum delay $D$, Elimination Threshold $B$.
\State \textbf{Initialization:} $t\gets1$, $S\gets[K]$, $n\gets0$
\While{$t<T$}
\State Play each arm $i\in S$
\State Observe any incoming payoff
\State Set $t\gets t+|S|$
\State Set $n \gets n+1$
\For{$i\in S$}
    \State $M_t(i)\gets \sett{s\in[t]}{i_s=i\land s+d_s\geq t}$\Comment{Plays that we are waiting for}
    \State $\obs_t(i)\gets \sett{s\in[t]}{i_s=i\land s+d_s< t}$\Comment{Plays that are completed}
    \State $m_t(i)\gets |M_t(i)|$\Comment{Number of missing plays}
    \State $\hat{\mu}_t^-(i)\gets \dfrac{1}{n}(\sum_{s \in M_t(i)}{\dfrac{t - s}{D}} + \sum_{s \in \obs_t(i)}{c_s})$
    \State $L_t^1(i) \gets \hat{\mu}_t^-(i)-\sqrt{\frac{2\log T}{n}}$
    \State $F_t(i) \gets \sett{s\in[t-D]}{i_s=i}$\Comment{Full information plays}
    \State $n_t^{F} \gets |F_t(i)|$ \Comment{Same for all $i\in S$}
    \State $\hat{\mu}_t^{F}(i)\gets \dfrac{1}{n_t^{F}}(\sum_{s\in F_t(i)}{c_s})$
    \State $L_t^2(i) \gets \hat{\mu}_t^{F}(i)-\sqrt{\frac{2\log T}{n_t^{F}}}$
    \State $L_t^3(i) \gets \dfrac{|S|}{D}\roundy{\dfrac{m_t(i)}{2} - 8\log T - 1}$
    \State $LCB_t(i)\gets \max\curly{L_t^1(i), L_t^2(i), L_t^3(i)}$
    \State $UCB_t(i)\gets \hat{\mu}_t^{F}(i)+\sqrt{\frac{2\log T}{n_t^{F}}}$
\EndFor
    \State Remove from $S$ any arm $i$ if there exists $j$ such that $\min\curly{UCB_t(j), B} < LCB_t(i)$
    \If{$S = \emptyset$}
        \State Return (\texttt{Fail}, $t$)
    \EndIf
\EndWhile
\State Return (\texttt{Success}, $t$)
\end{algorithmic}
\end{algorithm}

\begin{definition}
Let $G$ be the event that all of the below happens for every $t\in[T]$ and $i\in[K]$:
\begin{align*}
m_t(i) &\le \dfrac{2d(i)}{|S_t|} + 16\log T + 2\\
|\mu(i) - \hat{\mu}_t(i) | &\le \sqrt{\dfrac{2\log T}{n_t(i)}}
\end{align*}
\end{definition}

\begin{lemma}[\Cref{lem:good_g_prob} in the main paper] \label{lem:G prob cost}
$G$ holds w.p $1 - \frac{3}{T^2}$
\end{lemma}
\begin{proof}
Follows immediately from \cite{slivkins2024introductionmultiarmedbandits}[Equation 1.6] and \Cref{lem:w_upper_bound}.
\end{proof}

\begin{lemma}[\Cref{thm:CSE_cost_safe_elimination_mt} in the main paper] 
\label{lem:cse_safe_elimination}
In \Cref{alg:CSE} assume $G$ happens, if $B \ge \mu^*$ then for every $t \in [T]$, $i^* \in S_t$
\end{lemma}
\begin{proof}
Note that for any arm $i$, $\hat{\mu}_t^-(i)$ is an optimistic estimation of $\hat{\mu}_t(i)$, because it assumes for every missing play a cost of the amount of steps passed since the play, which will always be less than the actual cost, so following from the definition of $G$ we can say that $L_t^1(i) \le \mu(i)$. $L_t^3(i) \le \mu(i)$ also follow immediately from $G$. \\
Notice that $\hat{\mu}_t^{F}(i) = \hat{\mu}_{t-D}(i)$. Using the definition of $G$:
\begin{align*}
    L_t^2(i) &= \hat{\mu}_t^{F}(i) - \sqrt{\dfrac{2\log T}{n^{F}_t(i)}} \le \mu(i) \\
\end{align*}
Which means that:
\begin{align*}
    LCB_t(i^*) = \max\curly{L_t^1(i), L_t^2(i), L_t^3(i)} &\le \mu^* \leq B
\end{align*}
Again using $G$, for any $i$:
\begin{align*}
    UCB_t(i) &= \hat{\mu}_t^{F}(i) + \sqrt{\dfrac{2\log T}{n^{F}_t(i)}} \ge \mu(i) \geq \mu^* \geq LCB_t(i^*) \\
\end{align*}
Therefore,
\begin{align*}
    LCB_t(i^*) \le \min\curly{B, UCB_t(i)}
\end{align*}
for any arm $i$, which means that $i^*$ will not be eliminated.
\end{proof}

\begin{lemma} \label{lem:mu_r_bound}
In \Cref{alg:CSE}, assume $G$ happens. For any arm $i\neq i^*$:
\begin{align*} 
    \mu(i) \le B + 2\sqrt{\dfrac{2\log T}{n_{\tau_i}(i)}} + \dfrac{2DB}{K_in_{\tau_i}(i)} + \dfrac{16\log T + 2}{n_{\tau_i}(i)}
\end{align*}
\end{lemma}
\begin{proof}
Fix $t = \tau_i$. 
Since the difference between the known delay of a play to its real delay is bounded by $D$:
\begin{align*}
    \hat{\mu}_t^-(i) &=  \dfrac{1}{n_t(i)}\roundy{\sum_{s \in M_t(i)}{\dfrac{t - s}{D}} + \sum_{s \in \obs_t(i)}{c_s}} \\
    &\ge \dfrac{1}{n_t(i)}\roundy{\sum_{s \in M_t(i)}{\roundy{c_s - 1}} + \sum_{s \in \obs_t(i)}{c_s}} \\
    &\ge \dfrac{1}{n_t(i)}\roundy{\sum_{s \in M_t(i) \cup \obs_t(i)}{c_s} - m_t(i)} \\
    &= \hat{\mu}_t(i) - \dfrac{m_t(i)}{n_t(i)}
\end{align*}
Since $G$ happens:
\begin{align*}
    \mu(i) &\le \hat{\mu}_t(i) + \sqrt{\dfrac{2\log T}{n_t(i)}} \\
    \mu(i) &\le \hat{\mu}_t^-(i) + \sqrt{\dfrac{2\log T}{n_t(i)}} + \dfrac{m_t(i)}{n_t(i)}\\
    \mu(i) &\le LCB_t(i) + 2\sqrt{\dfrac{2\log T}{n_t(i)}} + \dfrac{m_t(i)}{n_t(i)}  \\
\end{align*}
Also, from the definition of $LCB$:
\begin{align*}
        LCB_t(i) &\geq L_t^3(i) \geq \dfrac{K_i}{D}\roundy{\dfrac{m_t(i)}{2} - 8\log T - 1} \\
        m_t(i) &\le \dfrac{2LCB_t(i)D}{K_i} + 16\log T + 2
\end{align*}

Since $i$ wasn't eliminated, it means that $LCB_t(i) \le B$, which means:
\begin{align*}
    \mu(i) &\le B + 2\sqrt{\dfrac{2\log T}{n_t(i)}} + \dfrac{m_t(i)}{n_t(i)}\\
    m_t(i) &\le \dfrac{2DB}{K_i} + 16\log T + 2\\
    \mu(i) &\le B + 2\sqrt{\dfrac{2\log T}{n_t(i)}} + \dfrac{2DB}{K_in_t(i)} + \dfrac{16\log T + 2}{n_t(i)}
\end{align*}
\end{proof}

\begin{lemma} \label{lem:regret_delta_bound}
In \Cref{alg:CSE}, assume $G$ happens. For any arm $i\neq i^*$:
\begin{align*}
    \Delta_i n_{\tau_i}(i) \le \dfrac{32\log T}{\Delta_i} + \dfrac{2\Delta_i D}{K_i} 
\end{align*}
\end{lemma}
\begin{proof}
Fix $t = \tau_i$. Since $G$ happens:
\begin{align*}
    L_t^2(i) + 2\sqrt{\dfrac{2\log T}{n_t^{F}(i)}} &\ge \mu(i) \\
    UCB_t(i) - 2\sqrt{\dfrac{2\log T}{n_t^{F}(i)}} &\le \mu(i) \\
\end{align*}
If $B \ge \mu^*$ , from \Cref{lem:cse_safe_elimination} $i^*$ was not eliminated, so since $i$ wasn't eliminated as well $L_t^2(i) \le UCB_t(i^*)$ and $n_t^{F}(i) = n_t^{F}(i^*)$. Thus:
\begin{align*}
    \mu(i) - 2\sqrt{\dfrac{2\log T}{n_t^{F}(i)}} &\le \mu^* + 2\sqrt{\dfrac{2\log T}{n_t^{F}(i^*)}} \\
    \Delta_i &\le 4\sqrt{\dfrac{2\log T}{n_t^{F}(i)}} \\
\end{align*}
If $B \le \mu^*$, since $B \ge L_t^2(i)$,
\begin{align*}
    \mu(i) - 2\sqrt{\dfrac{2\log T}{n_t^{F}(i)}} &\le B \\
    \mu(i) - 2\sqrt{\dfrac{2\log T}{n_t^{F}(i)}} &\le \mu^* \\
    \Delta_i &\le 2\sqrt{\dfrac{2\log T}{n_t^{F}(i)}}
\end{align*}
Which concludes in both cases to:
\begin{align}
    \Delta_i n_t^{F}(i) & \le \dfrac{32\log T}{\Delta_i}
    \label{eq:FI regret}
\end{align}
In the last $D$ steps, $i$ was played a maximum of $\left\lceil \dfrac{D}{K_i} \right\rceil$ times, which means:
\begin{align*}
    n_t(i) &\le n_t^{F}(i) + \dfrac{D}{K_i} + 1 \\
    n_t(i) &\le n_t^{F}(i) + \dfrac{2D}{K_i} \\
    \Delta_i n_t(i) &\le \Delta_i n_t^{F}(i) +  \dfrac{2\Delta_i D}{K_i} \\
    &\le \dfrac{32\log T}{\Delta_i} + \dfrac{2\Delta_i D}{K_i} 
\end{align*}
where the last is from \Cref{eq:FI regret}. 
\end{proof}

\begin{lemma}\label{lem:cse_action_regret}
In \Cref{alg:CSE}, assume $G$ happens. For any arm $i\neq i^*$:
\begin{align*}
    \Delta_i n_{\tau_i}(i) \le \dfrac{128\log T}{\Delta_i} + \dfrac{8DB}{K_i}
\end{align*}
\end{lemma}
\begin{proof}
Fix $t = \tau_i$. 
If $\mu(i) \le 2B$, it means that $\Delta_i \le 2B$, so from \Cref{lem:regret_delta_bound}:
\begin{align*}
    \Delta_i n_t(i) \le \frac{32\log T}{\Delta_i} + \frac{4DB}{K_i}
\end{align*}
Assume $\mu(i) \ge 2B$, From \Cref{lem:mu_r_bound}:
\begin{align*} 
    \mu(i) &\le \dfrac{\mu(i)}{2} + 2\sqrt{\dfrac{2\log T}{n_{\tau_i}(i)}} + \dfrac{2DB}{K_in_{\tau_i}(i)} + \dfrac{16\log T + 2}{n_{\tau_i}(i)} \\
    \Delta_i &\le \mu(i) \le \underbrace{ 4\sqrt{\dfrac{2\log T}{n}}}_{(i)} + \underbrace{\dfrac{4DB}{K_in_t(i)} + \dfrac{32\log T + 4}{n_t(i)}}_{(ii)}
\end{align*}
If $(i) \geq (ii)$, then $\Delta_i \le 8\sqrt{\dfrac{2\log T}{n_t}}$:
\begin{align*}
    \Delta_i^2 &\le \dfrac{128\log T}{n_t(i)} \\ 
    \Delta_i n_t(i) & \le \dfrac{128\log T}{\Delta_i}
\end{align*}
Else, $\Delta_i \le \dfrac{8DB}{K_in_t(i)} + \dfrac{64\log T + 8}{n_t(i)}$:
\begin{align*}
    \Delta_i n_t(i) \le \dfrac{8DB}{K_i} + 64\log T + 8
\end{align*}
Which concludes to:
\begin{align*}
    \Delta_i n_t(i) \le \dfrac{128\log T}{\Delta_i} + \dfrac{8DB}{K_i}
\end{align*}
\end{proof}

\begin{theorem}[\Cref{thm:CSE_bound} in the main paper]
\label{thm:cse regret}
The regret until failure of \Cref{alg:CSE} with elimination threshold $B$ is bounded by,
\begin{align*}
    \mathcal{R}_T \le \sum_{i:\Delta > 0}{\dfrac{129\log T}{\Delta_i}} + 8D\min\curly{B, \Delta_{max}}\log K
\end{align*}
\end{theorem}
\begin{proof}
From \Cref{lem:cse_action_regret} and \Cref{lem:regret_delta_bound}, for every $i\neq i^*$:
\begin{align*}
    \Delta_i n_{\tau_i}(i) \le \dfrac{128\log T}{\Delta_i} + \dfrac{8D\min\curly{B, \Delta_{max}}}{K_i}
\end{align*}
Using \Cref{lem:ending} we have the desired results.
\end{proof}

\subsection{Bounded Doubling Successive Elimination}
\begin{algorithm}
\caption{Bounded Doubling Successive Elimination}\label{alg:BDSE}
\begin{algorithmic}
\State \textbf{Input:} number of rounds $T$, number of arms $K$, maximum delay $D$.
\State \textbf{Initialization:} $B\gets\dfrac{1}{D}$
\While{$t<T$}
\State Run $(F,\tau) \gets \texttt{CSE}(T-t, K, D, B)$ (see \Cref{alg:CSE})
\If{F=Fail}
    \State $B \gets 2 * B$
    \State $t \gets t + \tau$
\EndIf
\EndWhile
\end{algorithmic}
\end{algorithm}
\begin{corollary}[\Cref{corr:bdse_regret} in the main paper]
The regret of \Cref{alg:BDSE} is bounded by,
\begin{align*}
    \mathcal{R}_T \le \sum_{i:\Delta > 0}{\dfrac{129\log T\log d^*}{\Delta_i}} + 8\min\curly{d^*, D\Delta_{max}\log d^*}\log K
\end{align*}
\end{corollary}
\begin{proof}
From \Cref{lem:cse_safe_elimination} we know that if $B \ge \mu^*$ \texttt{CSE} will not fail, which means that the loop will run a maximum of $\log d^*$ times. Notice that on the $j$'s run of \texttt{CSE}, $DB = 2^j$. By \Cref{thm:cse regret},
\begin{align*}
    \mathcal{R}_T &\le \sum_{j=0}^{\log d^*}{\roundy{\sum_{i:\Delta > 0}{\dfrac{129\log T}{\Delta_i}} + 8 \cdot 2^j\log K}} \\
    &= \sum_{i:\Delta > 0}{\dfrac{129\log T\log d^*}{\Delta_i}} + 8\log K\sum_{j=0}^{\log d^*}{2^j} \\
    &= \sum_{i:\Delta > 0}{\dfrac{129\log T\log d^*}{\Delta_i}} + 8d^*\log K \\
\end{align*}
Similarly, it can also be bounded by:
\begin{align*}
    \mathcal{R}_T &\le \sum_{j=0}^{\log d^*}{\roundy{\sum_{i:\Delta > 0}{\dfrac{129\log T}{\Delta_i}} + 8D\Delta_{max}\log K}} \\
    &\le \sum_{i:\Delta > 0}{\dfrac{129\log T\log d^*}{\Delta_i}} + 8D\Delta_{max}\log d^* \log K
\end{align*}
Which brings us to the desired bound.
\end{proof}

\section{Reward Equals Delay}
\begin{algorithm}
\caption{Reward Successive Elimination}\label{alg:RSE}
\begin{algorithmic}
\State \textbf{Input:} number of rounds $T$, number of arms $K$, maximum delay $D$, Elimination Threshold $B$.
\State \textbf{Initialization:} $t\gets1$, $S\gets[K]$, $n\gets0$
\While{$t<T$}
\State Play each arm $i\in S$
\State Observe any incoming payoff
\State Set $t\gets t+|S|$
\State Set $n \gets n+1$
\For{$i\in S$}
    \State $M_t(i)\gets \sett{s\in[t]}{i_s=i\land s+d_s\geq t}$\Comment{Plays that we are waiting for}
    \State $\obs_t(i)\gets \sett{s\in[t]}{i_s=i\land s+d_s< t}$\Comment{Plays that are completed}
    \State $m_t(i)\gets |M_t(i)|$\Comment{Number of missing plays}
    \State $\hat{\mu}_t^+(i)\gets \dfrac{1}{n}(\sum_{s \in M_t(i)}{1} + \sum_{s \in \obs_t(i)}{r_s})$
    \State $U_t^1(i) \gets \hat{\mu}_t^+(i)+\sqrt{\frac{2\log T}{n}}$
    \State $F_t(i) \gets \sett{s\in[t-D]}{i_s=i}$\Comment{Full information plays}
    \State $n_t^{F} \gets |F_t(i)|$ \Comment{Same for all $i\in S$}
    \State $\hat{\mu}_t^{F}(i)\gets \dfrac{1}{n_t^{F}}(\sum_{s\in F_t(i)}{r_s})$
    \State $U_t^2(i) \gets \hat{\mu}_t^F(i)+\sqrt{\frac{2\log T}{n}}$ 
    \State $UCB_t(i)\gets \min\curly{U_t^1(i), U_t^2(i)}$
    \State $LCB_t(i)\gets \hat{\mu}_t^{F}(i)-\sqrt{\frac{2\log T}{n_t^{F}}}$
\EndFor
    \State Remove from $S$ any arm $i$ if there exists $j$ such that $\max\curly{LCB_t(j), B} > UCB_t(i)$
    \If{$S = \emptyset$}
        \State Return (\texttt{Fail}, $t$)
    \EndIf
\EndWhile
\State Return (\texttt{Success}, $t$)
\end{algorithmic}
\end{algorithm}

\begin{definition}
Let $G$ be the event that all of the below happens for every $t\in[T]$ and $i\in[K]$:
\begin{align*}
m_t(i) &\le \dfrac{2d(i)}{|S_t|} + 16\log T + 2\\
|\mu(i) - \hat{\mu}_t(i) | &\le \sqrt{\dfrac{2\log T}{n_t(i)}}
\end{align*}
\end{definition}

\begin{lemma} \label{lem:G prob reward}
$G$ holds w.p $1 - \frac{3}{T^2}$
\end{lemma}
\begin{proof}
Follows immediately from \cite{slivkins2024introductionmultiarmedbandits}[Equation 1.6] and \Cref{lem:w_upper_bound}.
\end{proof}

\begin{lemma}[\Cref{thm:bse_reward_safe_elimination_mt} in the main paper] \label{lem:rbe_safe_elimination}
In \Cref{alg:RSE}, assume $G$ happens, if $B \le \mu^*$ then for every $t\in[T]$, $i^*\in S_t$
\end{lemma}

\begin{proof}
Notice that $\hat{\mu}_t^+(i)$ is an optimistic estimation of $\hat{\mu}_t(i)$, because it assumes for every missing play a cost of $1$, which is an upper bound for the actual cost, so under the good event $G$ we can say that $U_t^1(i) \ge \mu(i)$.

Recall that $\hat{\mu}_t^{F}(i) = \hat{\mu}_{t-D}(i)$. So, again using the definition of $G$:
\begin{align*}
    LCB_t(i) &= \hat{\mu}_t^{F}(i) - \sqrt{\dfrac{2\log T}{n^{F}_t}} \le \mu(i) \\
    U_t^2(i) &= \hat{\mu}_t^{F}(i) + \sqrt{\dfrac{2\log T}{n_t^{F}}} \ge \mu(i) \\
\end{align*}
Which means that $UCB_t(i^*) \ge \mu^*$. Thus:
\begin{align*}
    LCB_t(i) \le \mu(i) \le \mu^* \le UCB_t(i^*)
\end{align*}
Also, by the condition of the statement $B \le \mu^* \le UCB_t(i^*)$.
Which means that $i^*$ will not be eliminated.
\end{proof}

\begin{lemma}\label{lem:reward_r_bound}
In \Cref{alg:RSE}, assume $G$ happens. For any arm $i\neq i^*$:
\begin{align*}
    B \le \mu(i) + 2\sqrt{\dfrac{2\log T}{n_{\tau_i}(i)}} + \dfrac{2d(i)}{K_i n_{\tau_i}(i)} + \dfrac{16\log T + 2}{n_{\tau_i}(i)}
\end{align*}
\end{lemma}
\begin{proof}
Fix $t = \tau_i$. 
From the definition of $\hat{\mu}^+_t(i)$:
\begin{align*}
    \hat{\mu}^+_t(i) &= \dfrac{1}{n_t(i)}\roundy{\sum_{s \in M_t(i)}{1} + \sum_{s \in \obs_t(i)}{r_s}} \\
    &\le \dfrac{1}{n_t(i)}\roundy{\sum_{s \in M_t(i)}{(1 + r_s)} + \sum_{s \in \obs_t(i)}{r_s}} \\
    &\le \dfrac{1}{n_t(i)}\roundy{\sum_{s \in M_t(i)\cup \obs_t(i)}{r_s} + \sum_{s \in M_t(i)}{1}} \\
    &= \hat{\mu}_t(i) + \dfrac{m_t(i)}{n_t(i)}
\end{align*}
Since $G$ happens:
\begin{align*}
    \mu(i) &\ge \hat{\mu}_t(i) - \sqrt{\dfrac{2\log T}{n_t(i)}} \\
    &\ge \hat{\mu}^+_t(i) - \sqrt{\dfrac{2\log T}{n_t(i)}} - \dfrac{m_t(i)}{n_t(i)} \\
    &\ge UCB_t(i) - 2\sqrt{\dfrac{2\log T}{n_t(i)}} - \dfrac{m_t(i)}{n_t(i)} \\
    &\ge UCB_t(i) - 2\sqrt{\dfrac{2\log T}{n_t(i)}} - \dfrac{2d(i)}{K_i n_t(i)} - \dfrac{16\log T + 2}{n_t(i)} \\
\end{align*}
Since $i$ wasn't eliminated, $UCB_t(i) \ge B$, which means:
\begin{align*}
    B \le \mu(i) + 2\sqrt{\dfrac{2\log T}{n_t(i)}} + \dfrac{2d(i)}{K_i n_t(i)} + \dfrac{16\log T + 2}{n_t(i)}
\end{align*}    
\end{proof}

\begin{lemma} \label{lem:rse_delta_bound}
In \Cref{alg:RSE}, assume $G$ happens. For any arm $i\neq i^*$:
\begin{align*}
    \Delta_i n_{\tau_i} \le \dfrac{32\log T}{\Delta_i} + \dfrac{2\Delta_i D}{K_i}
\end{align*}
\end{lemma}
\begin{proof}
Fix $t = \tau_i$. 
Since $G$ happens:
\begin{align*}
    LCB_t(i) + 2\sqrt{\dfrac{2\log T}{n_t^{FI}(i)}} &\ge \mu(i) \\
    U_t^2(i) - 2\sqrt{\dfrac{2\log T}{n_t^{FI}(i)}} &\le \mu(i) \\
\end{align*}
If $B\le \mu^*$, from \Cref{lem:rbe_safe_elimination} $i^*$ was not eliminated, so since $i$ wasn't eliminated as well $LCB_t(i^*) \le U_t^{2}(i)$ and $n_t^{F}(i) = n_t^{F}(i^*)$. Thus:
\begin{align*}
    \mu^* - 2\sqrt{\dfrac{2\log T}{n_t^{F}(i)}} &\le \mu(i) + 2\sqrt{\dfrac{2\log T}{n_t^{F}(i^*)}} \\
    \Delta_i &\le 4\sqrt{\dfrac{2\log T}{n_t^{F}(i)}} \\
\end{align*}
If $B \ge \mu^*$, since $B \le U_t^2(i)$,
\begin{align*}
    \mu^* &\le B \le \mu(i) + 2\sqrt{\dfrac{2\log T}{n_t^{F}(i)}} \\
    \Delta_i &\le 2\sqrt{\dfrac{2\log T}{n_t^{F}(i)}}
\end{align*}
Which concludes in both cases to:
\begin{align}
    \Delta_i n_t^{F}(i) & \le \dfrac{32\log T}{\Delta_i}
    \label{eq:FI regret reward}
\end{align}
In the last $D$ steps, $i$ was played a maximum of $\left\lceil \dfrac{D}{K_i} \right\rceil$ times, which means:
\begin{align*}
    n_t(i) &\le n_t^{F}(i) + \dfrac{D}{K_i} + 1 \\
    n_t(i) &\le n_t^{F}(i) + \dfrac{2D}{K_i} \\
    \Delta_i n_t(i) &\le \Delta_i n_t^{F}(i) +  \dfrac{2\Delta_i D}{K_i} \\
    &\le \dfrac{32\log T}{\Delta_i} + \dfrac{2\Delta_i D}{K_i} 
\end{align*}
where the last is from \Cref{eq:FI regret reward}. \\
\end{proof}

\begin{lemma} \label{lem:rbe_action_regret}
In \Cref{alg:RSE}, assume $G$ happens and that $B \ge \dfrac{\mu^*}{2}$. Then:
\begin{align*}
    \Delta_i n_{\tau_i}(i) \le \dfrac{289\log T}{\Delta_i} + \dfrac{12d(i)}{K_i}
\end{align*}
\end{lemma}

\begin{proof}
Fix $t = \tau_i$. 
If $\mu(i) \ge \frac{B}{2}$, since $B \ge \frac{\mu^*}{2}$:
\begin{align*}
    \mu(i) &\ge \dfrac{B}{2} \ge \dfrac{\mu^*}{4} \\
    \Delta_i &= \mu^* - \mu(i) \le 4\mu(i) - \mu(i) = 3\mu(i)
\end{align*}
So using \Cref{lem:rse_delta_bound}:
\begin{align*}
    \Delta_i n_t(i) &\le \dfrac{32\log T}{\Delta_i} + \dfrac{6\mu(i)D}{K_i} \\
    &= \dfrac{32\log T}{\Delta_i} + \dfrac{6d(i)}{K_i}
\end{align*}
Else, namely $\mu(i) \le \frac{B}{2}$:
\begin{align*}
    B \ge 2\mu(i) \\
    2B - 2\mu(i) \ge B
\end{align*}
Since we assumed $B \ge \dfrac{\mu^*}{2}$:
\begin{align*}
    2B - 2\mu(i) &\ge B \ge \mu^* - B \\
    3B - 3\mu(i) &\ge \mu^* - B + B - \mu(i) = \Delta_i
\end{align*}
From \Cref{lem:reward_r_bound}:
\begin{align*}
    B - \mu(i) &\le 2\sqrt{\dfrac{2\log T}{n_t(i)}} + \dfrac{2d(i)}{K_i n_t(i)} + \dfrac{16\log T + 2}{n_t(i)} \\
    \Delta_i &\le \underbrace{6\sqrt{\dfrac{2\log T}{n_t(i)}}}_{(i)} + \underbrace{\dfrac{6d(i)}{K_i n_t(i)} + \dfrac{48\log T + 6}{n_t(i)}}_{(ii)}
\end{align*}
If $(i) \ge (ii)$, then $\Delta_i \le 12\sqrt{\dfrac{2\log T}{n_t(i)}}$:
\begin{align*}
    \Delta_i^2 &\le \dfrac{288\log T}{n_t(i)} \\
    \Delta_i n_t(i) &\le \dfrac{288\log T}{\Delta_i}
\end{align*}
Else, namely $\Delta_i \le \dfrac{12d(i)}{K_i n_t(i)} + \dfrac{96\log T + 12}{n_t(i)}$:
\begin{align*}
    \Delta_i n_t(i) \le \dfrac{12d(i)}{K_i} + 96\log T + 12
\end{align*}
Which concludes to:
\begin{align*}
    \Delta_i n_t(i) \le \dfrac{288\log T}{\Delta_i} + \dfrac{12d(i)}{K_i}
\end{align*}
\end{proof}

\begin{theorem}[\Cref{thm:rse_regret} in the main paper]
Assume $B \ge \frac{\mu^*}{2}$. The regret of \Cref{alg:RSE} is bounded by,
\begin{align*}
    \mathcal{R}_T \le \sum_{i:\Delta > 0}\dfrac{289\log T}{\Delta_i} + 12\min\curly{\bar{d}, D\Delta_{max}}\log K
\end{align*}
\end{theorem}
\begin{proof}
From \Cref{lem:rbe_action_regret} and the definition of $\bar{d}$, for every $i\in[K]$ if $G$ happens:
\begin{align*}
        \Delta_i n_{\tau_i} \le \dfrac{288\log T}{\Delta_i} + \dfrac{12\bar{d}}{K_i}
\end{align*}
Using also \Cref{lem:rse_delta_bound}:
\begin{align*}
    \Delta_i n_{\tau_i} \le \dfrac{288\log T}{\Delta_i} + \dfrac{12\min\curly{\bar{d}, D\Delta_{max}}}{K_i}
\end{align*}
Using \Cref{lem:ending} we get the desired results.
\end{proof}

\subsection{Bounded Halving Successive Elimination}
\begin{algorithm}
\caption{Bounded Halving Successive Elimination} \label{alg:BHSE}
\begin{algorithmic}
\State \textbf{Input:} number of rounds $T$, number of arms $K$, maximum delay $D$.
\State \textbf{Initialization:} $B\gets 1$
\While{$t<T$}
\State Run $(F, \tau) \gets \texttt{RSE}(T-t, K, D, B)$
\If{F=Fail}
    \State $B \gets B / 2$
    \State $t \gets t + \tau$
\EndIf
\EndWhile
\end{algorithmic}
\end{algorithm}

\begin{corollary}[\Cref{corr:BHSE_regret} in the main paper]
The regret of \Cref{alg:BHSE} is bounded by,
\begin{align*}
    \mathcal{R}_T \le \roundy{\sum_{i:\Delta > 0}\dfrac{289\log T}{\Delta_i} + 12\min\curly{\bar{d}, D\Delta_{max}}\log K}\log \dfrac{1}{\mu^*}
\end{align*}
\end{corollary}
\begin{proof}
From \Cref{lem:rbe_safe_elimination} we know that if $B \le \mu^*$ \texttt{RSE} will not fail, which means the loop will run a maximum of $\log \dfrac{1}{\mu^*}$ times. This also means that $B \ge \dfrac{\mu^*}{2}$, as needed. \\
Therefore, the total regret will be $\roundy{\sum_{i:\Delta > 0}\dfrac{289\log T}{\Delta_i} + 12\min\curly{\bar{d}, D\Delta_{max}}\log K}\log \dfrac{1}{\mu^*}$
\end{proof}

\section{Lower Bounds} \label{lb}
\begin{theorem}[\Cref{thm:cost_lb_mt} in the main paper]
In the cost scenario, for every choice of $d^* \le \frac{D}{2}$, there is an instance for which any algorithm will have a regret of $\Omega\roundy{d^*}$
\end{theorem}
\begin{proof}
Consider the following instance, for every $d^* \le \dfrac{D}{2}$:
\begin{align*}
    &x\sim\D^{\A}_1=d^*
    &x\sim\D^{\A}_2=D
\end{align*}
Until $t=d^*$, both arms are indistinguishable. Therefore, any algorithm will play arm 2 a min of $\dfrac{d}{2}$ times, which will cause are regret of $\roundy{1 - \mu^*}\dfrac{d^*}{2} \ge \dfrac{d^*}{4} = \Theta\roundy{d^*}$.
\end{proof}

\begin{theorem}[\Cref{thm:reward_lb_mt} in the main paper]
In the reward scenario, for every choice of $\bar{d} \le \frac{D}{2}$, there is an instance for which any algorithm will have a regret of $\Omega\roundy{\bar{d}}$
\end{theorem}

\begin{proof}
Consider the case of $K$ arms, where every arm $i$ returns a constant delay of $d(i)$, $d(1) \le d(2) \le ... \le d(k-1) \le d(k)$, where $d(k) = D$ and $d(k-1) \le \frac{D}{2}$. \\
By time $t = d(k-1)$, the arms $k$ and $k-1$ are indistinguishable. Therefore, any algorithm can't play one of them more then $\dfrac{t}{2}$ times. Since $\Delta_{k-1} = \Theta\roundy{1}$, for every $i\in[K-1]$ we can say $\Delta_i = \Theta(1)$. Since some sub optimal arms were played at least $\dfrac{d(k-1)}{2}$ times, the regret is $\Omega(d(k-1)) = \Omega(\bar{d})$. \\
\end{proof}

The following theorem shows that even if $\Delta$ is as small as $\mu^*$ the lower bound of $\Omega(d^*)$ still holds. Below we prove it for the cost case, but we note that for the reward case similar arguments give a lower bound of $\Omega(\bar{d})$.

\begin{theorem}
In the cost scenario, for every choice of $d^* \le \frac{D}{4}$ there is an instance with $\Delta=\Theta(\mu^*)$, such that any algorithm will have a regret of $\Omega\roundy{d^*}$. 
\end{theorem}
\begin{proof}
Consider the following instances for every $0 \le \mu \le 1$: 
\begin{align*}
    x\sim\D^{\A}_1&=\begin{cases}
        \dfrac{D}{4},&\text{w.p. } \mu\\
        0,&\text{otherwise}
    \end{cases}
    \qquad
    x\sim\D^{\A}_2=\dfrac{D\mu}{2} \\\\
    x\sim\D^{\B}_1&=\begin{cases}
        D,&\text{w.p. } \mu\\
        0,&\text{otherwise}
    \end{cases}
    \qquad
    x\sim\D^{\B}_2=\dfrac{D\mu}{2}
\end{align*}

We select one of the instances at random.

Notice that those instances are indistinguishable until time $t=\dfrac{D}{4}$. At that point, every algorithm will play one of the arms $\dfrac{D}{8}$ times. Importantly, the best arm differs between the instances. 

In the first instance $d^* = \frac{D\mu}{4}$ and $\Delta=\frac{D\mu}{4}$, so $\Delta=d^*/D$. In the second instance $d^* = \frac{D\mu}{2}$ and $\Delta=\frac{D\mu}{2}$, so also here $\Delta=d^*/D$. This means that we suffer a regret of $d^*/D$ in $\Theta\roundy{D}$ step, so any algorithm will have a regret of $\Theta\roundy{d^*}$.
\end{proof}

\section{Conservative Successive Elimination} \label{sec:regular successive elimination}

Following the discussion in \Cref{sec:clb}, we present a conservative SE algorithm that achieves the lower bound discussed in the main text. We show this for the cost scenario, but a similar algorithm can be constructed for the reward scenario.

\begin{algorithm}
\caption{Successive Elimination with Partial Knowledge}\label{alg:sepk}
\begin{algorithmic}
\State \textbf{Input:} number of rounds $T$, number of arms $K$, maximum delay $D$
\State \textbf{Initialization:} $t\gets1$, $S\gets[K]$, $n\gets 0$
\While{$t<T$}
\State Play each arm $i\in S$
\State Observe any incoming payoff
\State Set $t\gets t+|S|$
\State Set $n\gets n+1$
\For{$i\in S$}
    \State $M_t(i)\gets \sett{s\in[t]}{i_s=i\land s+d_s\geq t}$\Comment{Pulls that we are waiting for}
    \State $\obs_t(i)\gets \sett{s\in[t]}{i_s=i\land s+d_s< t}$\Comment{Pulls that are completed}
    \State $m_t(i) \gets |M_t(i)|$
    \State \State $\hat{\mu}_t^-(i)\gets \dfrac{1}{n}(\sum_{s \in M_t(i)}{\dfrac{t - s}{D}} + \sum_{s \in \obs_t(i)}{c_s})$\Comment{optimistic estimator for $\mu_i$}
    \State \State $\hat{\mu}_t^+(i)\gets \dfrac{1}{n}(\sum_{s \in M_t(i)}{1} + \sum_{s \in \obs_t(i)}{c_s})$\Comment{pessimistic estimator for $\mu_i$}
    \State $LCB_t(i)\gets \max\curly{\hat{\mu}_t^-(i)-\sqrt{\frac{2\log T}{t}}, \dfrac{|S|}{D}\roundy{\dfrac{m_t(i)}{2} - 8\log T - 1}}$
    \State $UCB_t(i)\gets \min\curly{\hat{\mu}_t^+(i)+\sqrt{\frac{2\log T}{t}}, \dfrac{2m_t(i) + 8\log T + 1}{n_t(i) - n_{t-D}(i)}}$
\EndFor
\State Remove from $S$ any arm $i$ if there exists $j$ such that $UCB_t (j) < LCB_t (i)$
\EndWhile
\end{algorithmic}
\end{algorithm}

\begin{definition}
Let $G$ be the event that all of the below happens for every $t\in[T]$ and $i\in[K]$:
\begin{align*}
m_t(i) &\ge \dfrac{\roundy{n_t(i) - n_{t-D}(i)}\mu(i) - 1}{2} - 4\log T\\
m_t(i) &\le \dfrac{2d(i)}{|S_t|} + 16\log T + 2\\
|\mu(i) - \hat{\mu}_t(i) | &\le \sqrt{\dfrac{2\log T}{n_t(i)}}
\end{align*}
\end{definition}

\begin{lemma} \label{lem:G prob sepk}
$G$ holds w.p $1 - \frac{4}{T^2}$
\end{lemma}
\begin{proof}
Follows immediately from \cite{slivkins2024introductionmultiarmedbandits}[Equation 1.6], \Cref{lem:w_upper_bound} and \Cref{lem:w_lower_bound}.
\end{proof}

\begin{lemma}[Safe Elimination] \label{lem:sepk_safe_elimination}
In \Cref{alg:sepk}, assume $G$, For every $t\in[T]$, $i^*\in S$.
\end{lemma}
\begin{proof}
Since $\hat{\mu}_t^+(i) \ge \mu_t(i) \ge \hat{\mu}_t^-(i)$, 
if $G$ holds we can say:
\begin{align*}
    \hat{\mu}_t^-(i) - \mu(i)  &\le \sqrt{\frac{2\log T}{t}} \\
    \mu(i) - \hat{\mu}_t^+(i)  &\le \sqrt{\frac{2\log T}{t}}
\end{align*}
Adding the bounds in $G$ for $m_t(i)$, we get:
\begin{align*}
    LCB_t(i) \le \mu(i) \le UCB_t(i)
\end{align*}
Which means that for every $i\in[K]$:
\begin{align*}
    LCB_t(i^*) \le \mu^* \le \mu(i) \le UCB_t(i)
\end{align*}
Which means that $i^*$ will not eliminated.
\end{proof}

\begin{lemma}[A Single Arm's Regret Contribution] \label{lem:sepk_arm_regret}
In \Cref{alg:sepk}, assume $G$ happens. for every action $i\in[K]$:
\begin{align*}
    \Delta_i n_{\tau_i}(i) \le \dfrac{128\log T}{\Delta_i} + \dfrac{8\sqrt{2Dd^*}}{K_i}
\end{align*}
\end{lemma}
\begin{proof} Denote $t=\tau_i$.
\begin{align*}
    \hat{\mu}_t^-(i) &=  \dfrac{1}{n_t(i)}\roundy{\sum_{s \in M_t(i)}{\dfrac{t - s}{D}} + \sum_{s \in \obs_t(i)}{f_s}} \\
    &\ge \dfrac{1}{n_t(i)}\roundy{\sum_{s \in M_t(i)}{\roundy{c_s - 1}} + \sum_{s \in \obs_t(i)}{c_s}} \\
    &\ge \dfrac{1}{n_t(i)}\roundy{\sum_{s \in M_t(i) \cup \obs_t(i)}{c_s} - m_t(i)} \\
    &= \hat{\mu}_t(i) - \dfrac{m_t(i)}{n_t(i)}\\
    \hat{\mu}^+_t(i) &= \dfrac{1}{n_t(i)}\roundy{\sum_{s \in M_t(i)}{1} + \sum_{s \in \obs_t(i)}{c_s}} \\
    &\le \dfrac{1}{n_t(i)}\roundy{\sum_{s \in M_t(i)}{(1 + c_s)} + \sum_{s \in \obs_t(i)}{c_s}} \\
    &\le \dfrac{1}{n_t(i)}\roundy{\sum_{s \in M_t(i)\cup \obs_t(i)}{c_s} + \sum_{s \in M_t(i)}{1}} \\
    &= \hat{\mu}_t(i) + \dfrac{m_t(i)}{n_t(i)}
\end{align*}
Since $G$ happens:
\begin{align*}
    \mu(i) &\le \hat{\mu}_t(i) + \sqrt{\dfrac{2\log T}{n_t(i)}} \\
    &\le \hat{\mu}_t^-(i) + \sqrt{\dfrac{2\log T}{n_t(i)}} + \dfrac{m_t(i)}{n_t(i)}\\
    &\le LCB_t(i) + 2\sqrt{\dfrac{2\log T}{n_t(i)}} + \dfrac{m_t(i)}{n_t(i)}  \\
    \mu(i) &\ge \hat{\mu}_t(i) - \sqrt{\dfrac{2\log T}{n_t(i)}} \\
    &\ge \hat{\mu}^+_t(i) - \sqrt{\dfrac{2\log T}{n_t(i)}} - \dfrac{m_t(i)}{n_t(i)} \\
    &\ge UCB_t(i) - 2\sqrt{\dfrac{2\log T}{n_t(i)}} - \dfrac{m_t(i)}{n_t(i)} \\
\end{align*}
From \Cref{lem:sepk_safe_elimination}, $i^*$ wasn't eliminated. Thus:
\begin{align}
    \notag LCB_t(i) &\le UCB_t(i^*) \\
    \notag \mu_t(i) - \dfrac{m_t(i)}{n_t(i)} - 2\sqrt{\dfrac{2\log T}{n_t(i)}} &\le \mu^* + \dfrac{m_t(i^*)}{n_t(i^*)} + 2\sqrt{\dfrac{2\log T}{n_t(i^*)}} \\
    \Delta_i &\le 4D\sqrt{\dfrac{2\log T}{n_t(i)}} + \dfrac{m_t(i)}{n_t(i)} + \dfrac{m_t(i^*)}{n_t(i)} \label{eq:sepk_delta_m_bound}
\end{align}

Additionally:
\begin{align*}
    \dfrac{K_i}{D}\roundy{\dfrac{m_t(i)}{2} - 8\log T - 1} &\le  \dfrac{2m_t(i^*) + 8\log T + 1}{n_t(i) - n_{t-D}(i)} \\
    m_t(i) &\le \underbrace{\frac{D}{K_i}\dfrac{2m_t(i^*) + 8\log T + 1}{n_t(i^*) - n_{t-D}(i^*)}}_{(i)} + \underbrace{8\log T + 1}_{(ii)}
\end{align*}
If $(i) \ge (ii)$, it means that:
\begin{align*}
     m_t(i) \le \frac{D}{K_i}\dfrac{4m_t(i^*) + 16\log T + 2}{n_t(i^*) - n_{t-D}(i^*)}
\end{align*}
The arms are played in a round robin manner, and the missing plays could only have been played in the last $D$ plays. Hence, we can say:
\begin{align*}
    m_t(i) \le n_t(i) - n_{t-D}(i) \le n_t(i^*) - n_{t-D}(i^*) + 1
\end{align*}
which means:
\begin{align*}
    m_t(i) - 1 &\le m_t(i) \le \frac{D}{K_i}\dfrac{4m_t(i^*) + 16\log T + 2}{m_t(i) - 1} \\
    m_t(i) &\le \sqrt{\frac{D\roundy{4m_t(i^*) + 16\log T + 2}}{K_i}} + 1
\end{align*}
Since $G$ holds:
\begin{align*}
    m_t(i) &\le \sqrt{\frac{D\roundy{\tfrac{8d^*}{K_i} + 64\log T + 8 + 16\log T + 2}}{K_i}} + 1\\
    &\le \frac{2\sqrt{2Dd^*}}{K_i} + \sqrt{80\log T + 10} + 1 
\end{align*}
If $(ii) \ge (i)$ it means that $m_t(i) \le 16\log T + 2$. So generally:
\begin{equation} \label{eq:sepk_final_m_bound}
    m_t(i) \le \frac{\sqrt{8Dd^*}}{K_i} + 16\log T + 2
\end{equation} 
From \Cref{eq:sepk_delta_m_bound,eq:sepk_final_m_bound} and $G$:
\begin{align*}
    \Delta_i &\le 4\sqrt{\dfrac{2\log T}{n_t(i)}} + \dfrac{m_t(i)}{n_t(i)} + \dfrac{m_t(i^*)}{n_t(i)} \\
    &\le 4\sqrt{\dfrac{2\log T}{n_t(i)}} + \frac{2\sqrt{2Dd^*}}{K_in_t(i)} + \frac{16\log T + 2}{n_t(i)} + \dfrac{m_t(i^*)}{n_t} \\
    &\le 4\sqrt{\dfrac{2\log T}{n_t(i)}} + \frac{2\sqrt{2Dd^*}}{K_in_t(i)} + \frac{16\log T + 2}{n_t(i)} + \frac{2d^*}{K_in_t(i)} + \frac{16\log T + 2}{n_t(i)} \\
    &= \underbrace{4\sqrt{\dfrac{2\log T}{n_t(i)}}}_{(i)} + \underbrace{\frac{4\sqrt{2Dd^*}}{K_in_t(i)} + \frac{32\log T + 4}{n_t(i)}}_{(ii)}
\end{align*}
If $(i) \ge (ii)$:
\begin{align*}
    \Delta_i &\le 8\sqrt{\dfrac{2\log T}{n_t(i)}} \\
    \Delta_in_t(i) &\le \frac{128\log T}{\Delta_i}
\end{align*}
Else:
\begin{align*}
    \Delta_i &\le \frac{8\sqrt{2Dd^*}}{K_in_t(i)} + \frac{64\log T + 8}{n_t(i)} \\
    \Delta_in_t(i) &\le \frac{8\sqrt{2Dd^*}}{K_i} + 64\log T + 8
\end{align*}
Which means that generally:
\begin{align*}
    \Delta_in_t(i) \le \frac{128\log T}{\Delta_i} + \frac{8\sqrt{2Dd^*}}{K_i}
\end{align*}
\end{proof}

\begin{theorem}
The regret of \Cref{alg:sepk} is bounded by,
\begin{align*}
    \mathcal{R}_T \le \sum_{i:\Delta > 0}{\frac{129\log T}{\Delta_i}} + 8\sqrt{2Dd^*}\log K
\end{align*}
\end{theorem}
\begin{proof}
    The proof follows immediately from \Cref{lem:sepk_arm_regret} and \Cref{lem:ending}.
\end{proof}

\section{Auxiliary lemmas}

\begin{lemma}[Lemma F.4 in \cite{dann2017unifying}]
    \label{lem:dann}
     Let $\{ X_t \}_{t=1}^T$ be a sequence of Bernoulli random and a filtration $\calF_1 \subseteq \calF_2 \subseteq...\calF_T$ with $\bbP(X_t = 1\mid \calF_t) = P_t$, $P_t$ is $\calF_{t}$-measurable and $X_t$ is $\calF_{t+1}$-measurable. Then, for all $t\in [T]$ simultaneously, with probability $1-\delta$,
     \[
        \sum_{k=1}^t X_k \geq \frac{1}{2}\sum_{k=1}^t P_k -\log \frac{1}{\delta}.
     \]
\end{lemma}

\begin{lemma}[Consequence of Freedman’s Inequality, e.g., Lemma E.2 in \cite{cohen2021minimax}]
    \label{lem:cons-freedman}
     Let $\{ X_t \}_{t\geq 1}$ be a sequence of random variables, supported in $[0,R]$, and adapted to a filtration $\calF_1 \subseteq \calF_2 \subseteq...\calF_T$. For any $T$, with probability $1-\delta$,
     \[
        \sum_{t=1}^T X_t \leq 2 \bbE[X_t \mid \calF_t] + 4R \log\frac{1}{\delta}.
     \]
\end{lemma}

\end{document}